\documentclass{article}

\newif\ifneurips
\neuripsfalse

\newif\ifchecklist
\checklistfalse

\newif\ifappendix
\appendixtrue
\newcommand{\appendixref}[1]{\ifappendix Appendix~\ref{#1}\else the appendix\fi}

\ifneurips
\PassOptionsToPackage{longnamesfirst}{natbib}

\usepackage[final]{neurips_2022}

\else
\usepackage[margin=1.1in]{geometry}
\usepackage{palatino}
\usepackage[longnamesfirst]{natbib}
\fi

\usepackage[utf8]{inputenc} 
\usepackage[T1]{fontenc}    
\usepackage[pagebackref,colorlinks,linkcolor=purple,citecolor=teal,urlcolor=violet]{hyperref}       
\usepackage{url}            
\usepackage{booktabs}       
\usepackage{amsfonts}       
\usepackage{nicefrac}       
\usepackage{microtype}      
\usepackage{xcolor}         

\usepackage{amsthm,thmtools,thm-restate}
\usepackage{doi}
\usepackage[shortlabels]{enumitem}
\SetLabelAlign{parright}{\parbox[t]{\labelwidth}{\raggedleft #1}}
\usepackage{framed}
\usepackage{graphicx,graphbox}
\usepackage{mathtools}
\usepackage{mismath}

\usepackage{tocloft}
\ifneurips
\else
\usepackage{parskip}
\fi

\usepackage{csvsimple}
\usepackage{pgfplotstable}
\usepgfplotslibrary{groupplots}
\pgfplotsset{compat=1.9} 

\usepackage{siunitx}
\newcommand\percentage[2][round-precision = 1]
{\SI[round-mode = places, 
      scientific-notation = fixed,fixed-exponent = 0,
      output-decimal-marker={.}, #1]{#2e2}{}}

\allowdisplaybreaks[1]
\renewcommand{\vec}[1]{\boldsymbol{#1}}

\declaretheorem{proposition}
\declaretheorem[sibling=proposition]{lemma}
\declaretheorem[sibling=proposition]{theorem}
\declaretheorem[sibling=proposition]{corollary}

\title{Adversarial Reprogramming Revisited}

\author{
Matthias Englert\thanks{Equal contribution.} \\
University of Warwick \\
\texttt{m.englert@warwick.ac.uk}
\ifneurips
\And
\else
\and
\fi
Ranko Lazi\'c\footnotemark[\value{footnote}] \\
University of Warwick \\
\texttt{r.s.lazic@warwick.ac.uk}
}

\ifneurips
\else
\date{}
\fi

\begin{document}

\maketitle

\begin{abstract}
Adversarial reprogramming, introduced by Elsayed, Goodfellow, and Sohl-Dickstein, seeks to repurpose a neural network to perform a different task, by manipulating its input without modifying its weights.  We prove that two-layer ReLU neural networks with random weights can be adversarially reprogrammed to achieve arbitrarily high accuracy on Bernoulli data models over hypercube vertices, provided the network width is no greater than its input dimension.  We also substantially strengthen a recent result of Phuong and Lampert on directional convergence of gradient flow, and obtain as a corollary that training two-layer ReLU neural networks on orthogonally separable datasets can cause their adversarial reprogramming to fail.  We support these theoretical results by experiments that demonstrate that, as long as batch normalisation layers are suitably initialised, even untrained networks with random weights are susceptible to adversarial reprogramming.  This is in contrast to observations in several recent works that suggested that adversarial reprogramming is not possible for untrained networks to any degree of reliability.
\end{abstract}

\section{Introduction}
\label{s:intro}

\citet{elsayed2018adversarial} proposed \emph{adversarial reprogramming}: given a neural network~$\mathcal{N}$ which performs a task~$F: X \to Y$ and given an adversarial task $G: \widetilde{X} \to \widetilde{Y}$, repurpose~$\mathcal{N}$ to perform~$G$ by finding mappings $h_{\mathrm{in}}: \widetilde{X} \to X$ and $h_{\mathrm{out}}: Y \to \widetilde{Y}$ such that $G \approx h_{\mathrm{out}} \circ F \circ h_{\mathrm{in}}$.  They focused on a setting where $F$ and~$G$ are image classification tasks, $X$~consists of large images, $\widetilde{X}$~consists of small images, and $h_{\mathrm{in}}$ and $h_{\mathrm{out}}$ are simple and computationally inexpensive: $h_{\mathrm{in}}(\widetilde{x}) = p + \widetilde{x}$ draws the input~$\widetilde{x}$ at the centre of the \emph{adversarial program}~$p$, and $h_{\mathrm{out}}$ is a hard coded mapping from the class labels~$Y$ to the class labels~$\widetilde{Y}$.  Then the challenge of adversarial reprogramming is to find~$p$ such that $h_{\mathrm{out}} \circ F \circ h_{\mathrm{in}}$ approximates well the adversarial task~$G$.

Remarkably, \citet{elsayed2018adversarial} showed that a range of realistic neural networks can be adversarially reprogrammed to perform successfully several tasks.  They considered six architectures trained on ImageNet~\citep{RussakovskyDSKS15} and found adversarial programs that achieve very good accuracies on a counting task, MNIST~\citep{LeCunBBH98}, and CIFAR-10~\citep{Krizhevsky09}.  In addition to the basic setting, they investigated limiting the visibility of the adversarial program by restricting its size or scale, or even concealing both the input and the adversarial program by hiding them within a normal image from ImageNet, and obtained good results.

Adversarial reprogramming can be seen as taking the crafting of adversarial examples~\citep{BiggioCMNSLGR13,SzegedyZSBEGF13} to a next level.  Like the universal adversarial perturbation of \citet{Moosavi-Dezfooli17}, a single adversarial program is combined with every input from the adversarial task, but in contrast to the former where the goal is to cause natural images to be misclassified with high probability, adversarial reprogramming seeks a high accuracy of classification for the given adversarial task.

Adversarial reprogramming is also related to transfer learning~\citep{RainaBLPN07,MesnilDGRBGLMDWVCB12}, however with important differences.  Whereas transfer learning methods use the knowledge obtained from one task as a base for learning to perform another, and allow model parameters to be changed for the new task, in adversarial reprogramming the new task may bear no similarity with the old, and the model cannot be altered but has to be manipulated through its input.

The latter features are suggestive of potential nefarious \label{p:nefarious} uses of adversarial reprogramming, and \citet{elsayed2018adversarial} list several, such as repurposing computational resources to perform a task which violates the ethics code of system providers.  However, its further explorations have demonstrated utility for virtuous deployments to medical image classification~\citep{TsaiCH20}, natural language processing~\citep{HambardzumyanKM20,NeekharaHDK19,NeekharaHDDKM22}, molecule toxicity prediction~\citep{VinodCD20}, and time series classification~\citep{YangTC21}.  In the last three works, adversarial reprogramming was achieved between different domains, e.g.\ repurposing neural networks trained on ImageNet to perform classification of DNA sequences, and of natural language sentiments and topics.

In spite of the variety of fruitful applications, it is still largely a mystery when adversarial reprogramming is possible and why.  In the only work in this direction, \citet{ZhengFXJDPBR21} proposed the alignment $\|g\|_1 / (\nicefrac{1}{n} \sum_i \|g_i\|_1)$ of the gradients~$g_i$ of the inputs from the adversarial task with their average~$g$ as the main indication of whether adversarial reprogramming will succeed.  However, their experiments did not show a significant correlation between accuracy after reprogramming and the gradient alignment before reprogramming.  The correlation was statistically significant with the gradient alignment after reprogramming, but that is unsurprising and it is unclear how to use it for predicting success.

Specifically, a central question on adversarial reprogramming is:
\begin{framed}
Can neural networks with random weights be adversarially reprogrammed, and more generally, how does training impact adversarial reprogrammability?
\end{framed}
First, addressing this question is important for assessing scope of two claims made in the literature:
\begin{description}[style=unboxed,wide]
\item[\emph{``[Adversarial] reprogramming usually fails when applied to untrained networks''} \citep{ZhengFXJDPBR21}.]
In addition to networks trained on ImageNet, \citet{elsayed2018adversarial} experimented using the same architectures untrained, i.e.\ with random weights, and obtained generally poor results.  Similarly, the experimental results of \citet{NeekharaHDK19} and \citet{ZhengFXJDPBR21} for random networks are significantly worse than their experimental results for trained networks.  However, they remarked that this was surprising given the known richness of random networks (see e.g.~\citet{HeWH16,LeeBNSPS18}), and that it was possibly due to simple reasons such as poor scaling of the random weights.
\item[\emph{``The original task the neural networks perform is important for adversarial reprogramming''} \citep{elsayed2018adversarial}.]
Nevertheless a number experimental results including those of \citet{TsaiCH20,NeekharaHDDKM22,VinodCD20,YangTC21,ZhengFXJDPBR21} demonstrated successful adversarial reprogramming between tasks that are seemingly unrelated (e.g.~repurposing networks trained on ImageNet to act as HCL2000~\citep{ZhangGCL09} classifiers) or from different domains, so the interaction between the original and adversarial tasks remains unclear.
\end{description}
Second, the question above is important because of implications for the following two applied considerations:
\begin{description}[style=unboxed,wide]
\item[Disentangling architecture and training as factors in adversarial reprogrammability.]
Clarifying the respective bearings on adversarial reprogramming success of the network architecture and of the task (if any) it was trained for would improve decision making, either to maximise the success in beneficial scenarios or to minimise it in detrimental ones.
\item[Managing the cost of adversarial reprogramming.]
Understanding when training the network is not essential, and when training for longer does not help or even hinders adversarial reprogrammability, should make it possible to control better the economic and environmental costs.
\end{description}

\subsection{Our contributions}

We initiate a theoretical study of adversarial reprogramming.  In it we focus on two-layer neural networks with ReLU activation, and on adversarial tasks given by the Bernoulli distributions of \citet{SchmidtSTTM18} over hypercube vertices.  The latter are binary classification data models in which the two classes are represented by opposite hypercube vertices, and when sampling a data point for a given class, we flip each coordinate of the corresponding class vertex with a certain probability.  This probability is a parameter that controls the difficulty of the classification task.  These data models are inspired by the MNIST dataset because MNIST images are close to binary (many pixels are almost fully black or white).  The remaining parameters are the radius of the hypercube, the input dimension of the neural network, and its width.

We prove that, in this setting, for networks with random weights, adversarial programs exist that achieve arbitrarily high accuracy on the Bernoulli adversarial tasks.  This holds for a wide variety of parameter regimes, provided the network width is no greater than its input dimension.  The adversarial programs we construct depend on the weights of the network and on the class vertices (i.e.\ the direction) of the Bernoulli data model, and their Euclidean length is likely to be close to the square root of the input dimension.  We present these results in Section~\ref{s:random}.

We also prove that, in the same setting, training the network on orthogonally separable datasets can cause adversarial reprogramming to fail.  \citet{PhuongL21} recently showed that, under several assumptions, training a two-layer ReLU network on such datasets by gradient flow makes it converge to a linear combination of two maximum-margin neurons; and subsequently \citet{WangP22} obtained in a different manner the same conclusion under the same assumptions.  We provide a simpler proof of a significantly stronger result: we show that the assumptions in \citet{PhuongL21} and \citet{WangP22} of small initialisation, and of positive and negative support examples spanning the whole space, are not needed; and we generalise to the exponential loss function as well as the logistic one.  We then observe that, for any Bernoulli data model whose direction is in a half-space of the difference of the maximum-margin neurons, and for any adversarial program, the accuracy tends to~$1 / 2$ (i.e.~approaches guessing) under a mild assumption on the growth rate of the difficulty of the data model.  We present these results in Section~\ref{s:bias}.

Both for the neworks with random weights and for the networks trained to infinity on orthogonally separable datasets, we then show that similar theoretical results can be obtained with a different kind of adversarial task, namely those given by the Gaussian distributions also of \citet{SchmidtSTTM18}.  The latter are mixtures of two spherical multivariate Gaussians, one per data class.  Please see \appendixref{app:Gaussian}.

In the experimental part of our work, we demonstrate that, as long as batch normalisation layers are suitably initialised, even untrained networks with random weights are susceptible to adversarial reprogramming.  Both the random weights and the batch normalisation layers are kept fixed throughout the finding of adversarial programs and their evaluation.  Our experiments are conducted with six realistic network architectures and MNIST as the adversarial task.  We investigate two different schemes to combine input images with adversarial programs: replacing the centre of the program by the image as was done by \citet{elsayed2018adversarial}, and scaling the image to the size of the program and then taking a convex combination of the two.  Each of the two schemes has a ratio parameter, and we explore their different values.  We find that the second scheme gives better results in our experiments, and that for some choices of the ratio parameter, the accuracies on the test set across all six architectures are not far below what \citet{elsayed2018adversarial} reported for networks trained on ImageNet.  Please see Section~\ref{s:exper}.

We conduct the same experiments also on the more challenging Fashion-MNIST~\citep{XiaoRV17} and Kuzushiji-MNIST~\citep{ClanuwatBKLYH18} datasets, and obtain broadly similar results, however with lower test accuracies in several cases; they are reported in \appendixref{app:other}.

For a further discussion of relations with other works, please see \appendixref{app:related}.

We conclude the paper in Section~\ref{s:concl}, where we discuss limitations of our work and suggest directions for future work.

\section{Random networks}
\label{s:random}

\paragraph{Basic notations.}

We write:
$[n]$~for the set $\{1, \ldots, n\}$,
$\|\vec{v}\|$~for the Euclidean length of a vector~$\vec{v}$,
$\angle(\vec{v}, \vec{v}')$ for the angle between~$\vec{v}$ and~$\vec{v}'$\ifappendix , \else , and \fi
$\mathbb{H}^d$~for the $d$-dimensional unit hypercube $\{\pm 1 / \sqrt{d}\}^d$\ifappendix , and
$\mathbb{S}^{d - 1}$~for the $d$-dimensional unit sphere $\{\vec{v} \in \mathbb{R}^d \,\mid\, \|\vec{v}\| = 1\}$\else\fi.

\paragraph{Two-layer ReLU networks.}

We consider two-layer neural networks~$\mathcal{N}$ with the ReLU activation.  We write $d$~for the input dimension, $k$~for the width, $\vec{w}_1, \ldots, \vec{w}_k \in \mathbb{R}^d$ for the weights of the first layer, and $a_1, \ldots, a_k \in \mathbb{R}$ for the weights of the second layer.  For an input $\vec{x} \in \mathbb{R}^d$, the output is thus
\[\mathcal{N}(\vec{x}) \coloneqq \sum_{j = 1}^k a_j \psi(\vec{w}_j^\top \vec{x})\;,\]
where $\psi(u) = \max \{u, 0\}$ is the ReLU function.

\paragraph{Random weights.}

In this section, we assume that the weights in~$\mathcal{N}$ are random as follows:
\begin{itemize}
\item
each $\vec{w}_j$ consists of $d$ independent centred Gaussians with variance~$1 / d$, and
\item
each $a_j$ is independently uniformly distributed in $\{\pm 1 / \sqrt{k}\}$.
\end{itemize}
This distribution is as in \citet{BubeckCGT21}, and standard for theoretical investigations.  It is similar to He's initialisation~\citep{HeZRS15}, with the second layer discretised for simplicity.

The variances of the weights are such that, for any input $\vec{x} \in \mathbb{R}^d$ of Euclidean length~$\sqrt{d}$, each $\vec{w}_j^\top \vec{x}$ is a standard Gaussian, and for large widths~$k$ the distribution of~$\mathcal{N}(\vec{x})$ is close to centred Gaussian with variance~$1 / 2$.

\paragraph{Bernoulli data models.}

Adapting from \citet{SchmidtSTTM18}, given a hypercube vertex $\vec{\phi} \in \mathbb{H}^d$, a radius $\rho > 0$, and a class bias parameter $0 < \tau \leq 1 / 2$, we define the $(\vec{\phi}, \rho, \tau)$-Bernoulli distribution over $(\vec{x}, y) \,\in\, \rho \mathbb{H}^d \times \{\pm 1\}$ as follows:
\begin{itemize}
\item
first draw the label~$y$ uniformly at random from $\{\pm 1\}$,
\item
then sample the data point~$\vec{x}$ by taking $y \rho \vec{\phi}$ and flipping the sign of each coordinate independently with probability $1 / 2 - \tau$.
\end{itemize}

These binary classification data models on hypercube vertices are inspired by the MNIST dataset~\citep{LeCunBBH98}.  The class bias parameter~$\tau$ controls the difficulty of the classification task, which increases as $\tau$~tends to zero, i.e.~as $1 / \tau$ tends to infinity.

In this section and the next, we consider adversarial tasks that are $(\vec{\phi}, \rho, \tau)$-Bernoulli data models, and investigate variations of the parameters $\vec{\phi}$, $\rho$ and~$\tau$.

\paragraph{Adversarial program.}

In this section, we assume that $k \leq d$, i.e.~the network width is no greater than the input dimension, and we define an adversarial program~$\vec{p}$ which depends on the weights of the network~$\mathcal{N}$ and on the direction~$\vec{\phi}$ of the Bernoulli data model.

With probability~$1$, for all $j \in [k]$, we have that $a_j \vec{w}_j^\top \vec{\phi} \neq 0$.  Let us write $K^+$~for $\{j \in [k] \,\mid\, a_j \vec{w}_j^\top \vec{\phi} > 0\}$, and $K^-$~for $\{j \in [k] \,\mid\, a_j \vec{w}_j^\top \vec{\phi} < 0\}$.  Then, for all $j \in [k]$, let:
\[\vec{p}'_j =
  \begin{cases}
  0                 & \text{if $j \in K^+$,} \\
  -\sqrt{d / |K^-|} & \text{if $j \in K^-$.}
  \end{cases}\]

Since $k \leq d$, with probability~$1$, the weights vectors~$\vec{w}_j$ are linearly independent, i.e.~the $k \times d$ matrix~$\vec{W}$ whose rows are~$\vec{w}_j$ has a positive smallest singular value $s_{\mathrm{min}}(\vec{W})$.  Hence $\vec{p} \in \mathbb{R}^d$ exists such that
\begin{equation}
\vec{p}' = \vec{W} \vec{p} \quad\text{and}\quad
\frac{\|\vec{p}'\|}{s_{\mathrm{max}}(\vec{W})}
\leq \|\vec{p}\| \leq
\frac{\|\vec{p}'\|}{s_{\mathrm{min}}(\vec{W})}\;.
\label{eq:p.p'}
\end{equation}

The neurons in~$K^+$ can be thought of as ``helpful'' for the adversarial task, and those in~$K^-$ as ``unhelpful''.  This definition of an adversarial program~$\vec{p}$ ensures that its effect is to introduce as first-layer biases the entries of the vector~$\vec{p}'$.  They are~$0$ (i.e.~do nothing) for every ``helpful'' neuron, and the negative value $-\sqrt{d / |K^-|}$ (i.e.~reduce the contribution to the network output) for every ``unhelpful'' neuron.

From the inequalities in \eqref{eq:p.p'}, we can expect $\|\vec{p}\| \approx \sqrt{d}$ if $k = o(d)$ (see \appendixref{app:p.p'} for details).

\paragraph{Expected reprogramming accuracy.}

The accuracy for an adversarial task~$\mathcal{D}$ is the probability that the sign of the output of the reprogrammed network conforms to the input label, i.e.
\[\mathbb{P}_{(\vec{x}, y) \sim \mathcal{D}}
  \{y \, \mathcal{N}(\vec{p} + \vec{x}) > 0\}\;.\]

Our main result in this section is that, for sufficiently large input dimensions~$d$, and under mild restrictions on the growth rates of the network width~$k$, the radius~$\rho$ and the difficulty~$1 / \tau$ of the Bernoulli data model, in expectation over the random network weights, the reprogramming accuracy is at least $(1 - C_1 \gamma) (1 - \gamma^\dag)$, which by tuning the probability parameters~$\gamma$ and~$\gamma^\dag$ can be arbitrarily close to~$100\%$.  The growth rate restrictions indicate that networks with smaller widths can be reprogrammed for tasks with larger radii, but that networks with larger widths can be reprogrammed for tasks that are more difficult.  The proof proceeds by analysing concentrations of the underlying distributions and involves establishing a theorem that bounds the reprogrammed network output; the details can be found in \appendixref{app:random}.

\begin{restatable}{corollary}{firstcorrandom}
\label{cor:random}
Suppose that
\[k        = \Theta(d^{\eta_{(k)}})\;, \quad
  \rho     = O(d^{\eta_{(\rho)}})\;, \quad
  1 / \tau = O(d^{\eta_{(\tau)}}) \quad\text{and}\quad
  1 / \tau = \omega_d(1)\;,\]
where $\eta_{(k)}, \eta_{(\rho)}, \eta_{(\tau)} \in [0, 1]$ are arbitrary constants that satisfy
\[\eta_{(\rho)} < 1 - \eta_{(k)} / 2 \quad\text{and}\quad
  \eta_{(\tau)} < \eta_{(k)} / 2\;.\]
Then, for sufficiently large input dimensions~$d$, the expected accuracy of the adversarially reprogrammed network~$\mathcal{N}$ on the $(\vec{\phi}, \rho, \tau)$-Bernoulli data model is arbitrarily close to~$100\%$.
\end{restatable}

\section{Implicit bias}
\label{s:bias}

\paragraph{Orthogonally separable dataset.}

In this section, we consider training the network on a binary classification dataset $S \,=\, \{(\vec{x}_1, y_1), \ldots, (\vec{x}_n, y_n)\} \,\subseteq\, \mathbb{R}^d \times \{\pm 1\}$ which is orthogonally separable~\citep{PhuongL21}, i.e.\ for all $i, i' \in [n]$ we have:
\[\vec{x}_i^\top \vec{x}_{i'} >    0 \quad\text{if}\quad y_i =    y_{i'}\;,
  \quad\text{and}\quad
  \vec{x}_i^\top \vec{x}_{i'} \leq 0 \quad\text{if}\quad y_i \neq y_{i'}\;.\]
In other words, every data point can act as a linear separator, although some data points from the opposite class may be exactly orthogonal to it.

\paragraph{Gradient flow with exponential or logistic loss.}

For two-layer ReLU networks with input dimension~$d$ and width~$k$ as before (but without the assumption $k \leq d$, which is not needed in this section), we denote the vector of all weights by
\[\vec{\theta} \coloneqq
  (\vec{w}_1, \ldots, \vec{w}_k, a_1, \ldots, a_k) \in
  \mathbb{R}^{k (d + 1)}\;,\]
and we write $\mathcal{N}_{\vec{\theta}}$ for the network whose weights are the coordinates of the vector~$\vec{\theta}$.

The empirical loss of~$\mathcal{N}_{\vec{\theta}}$ on~$S$ is
$\mathcal{L}(\vec{\theta}) \coloneqq
 \sum_{i = 1}^n \ell(y_i \mathcal{N}_{\vec{\theta}}(\vec{x}_i))$,
where $\ell$~is either the exponential $\ell_\mathrm{exp}(u) = e^{-u}$ or the logistic $\ell_\mathrm{log}(u) = \ln(1 + e^{-u})$ loss function.

A trajectory of gradient flow is a function $\vec{\theta}(t): [0, \infty) \to \mathbb{R}^{k (d + 1)}$ that is an arc, i.e.~it is absolutely continuous on every compact subinterval, and that satisfies the differential inclusion
\[\frac{\mathrm{d} \vec{\theta}}{\mathrm{d} t} \in -\partial \mathcal{L}(\vec{\theta}(t))
  \quad\text{for almost all}\quad t \in [0, \infty)\;,\]
where $\partial \mathcal{L}$ denotes the Clarke subdifferential~\citep{clarke1975generalized} of the locally Lipschitz function~$\mathcal{L}$.

Gradient flow is gradient descent with infinitesimal step size.  We work with the Clarke subdifferential in order to handle the non-differentiability of the ReLU function at zero: $\partial \psi(0)$ is the whole interval~$[0, 1]$.  At points of continuous differentiability, the Clarke subdifferential amounts to the gradient, e.g.~$\partial \mathcal{L}(\vec{\theta}) = \{\nabla \mathcal{L}(\vec{\theta})\}$.  For some further background, see \appendixref{app:Clarke}.

\paragraph{Initialisation of network weights.}

In this section, we assume that the initialisation is
\begin{description}[style=unboxed,wide]
\item[balanced:]
for all $j \in [k]$, at time $t = 0$ we have $|a_j| = \|\vec{w}_j\| > 0$; and
\item[live:]
for both signs $s \in \{\pm 1\}$ there exist $i_s \in [n]$ and $j_s \in [k]$ such that $y_{i_s} = s$ and at time $t = 0$ we have $y_{i_s} a_{j_s} \psi(\vec{w}_{j_s}^\top \vec{x}_{i_s}) > 0$.
\end{description}

The balanced assumption has featured in previous work (see e.g.~\citet{PhuongL21,LyuLWA21}).  It ensures that it remains to hold throughout the training, and that the signs of the second-layer weights~$a_j$ do not change (see the proof of Theorem~\ref{th:loss}).  The live assumption (present probabilistically in \citet{PhuongL21}) is mild: it states that at least one positively initialised neuron is active for at least one positive input, and the same for negative ones.

\paragraph{Convergence of gradient flow.}

Our main result in this section establishes that the early phase of training necessarily reaches a point where the empirical loss is less than~$\ell(0)$, which implies that then every input is classified correctly by the network.  Perhaps surprisingly, no small initialisation assumption is needed, however the proof makes extensive use of orthogonal separability of the dataset (see \appendixref{app:bias}, which contains all proofs for this section).

\begin{restatable}{theorem}{firstthloss}
\label{th:loss}
There exists a time~$t_0$ such that $\mathcal{L}(\vec{\theta}(t_0)) < \ell(0)$.
\end{restatable}

This enables us to apply to the late phase recent results of \citet{LyuL20,JiT20,LyuLWA21} and obtain the next corollary, which is significantly stronger than the main result of \citet{PhuongL21}, extending it to exponential loss, and showing that assumptions of small initialisation, and of positive and negative support examples spanning the whole space, are not needed.  The corollary establishes that each neuron converges to one of three types: a scaling of the maximum-margin vector for the positive data class, a scaling of the maximum-margin vector for the negative data class, or zero.  The two maximum-margin vectors are defined as follows: for both signs $s \in \{\pm 1\}$, let $I_s \coloneqq \{i \in [n] \,\mid\, y_i = s\}$, and let $\vec{v}_s$~be the unique minimiser of the quadratic problem
\[\text{\rm minimise}\quad
  \frac{1}{2} \|\vec{v}\|^2
  \quad\text{\rm subject to}\quad
  \forall i \in I_s: \, \vec{v}^\top \vec{x}_i \geq 1\;.\]
That a trajectory $\vec{\theta}(t)$ converges in direction to a vector $\widetilde{\vec{\theta}}$ means $\lim_{t \to \infty} \vec{\theta}(t) / \|\vec{\theta}(t)\| = \widetilde{\vec{\theta}} / \|\widetilde{\vec{\theta}}\|$.

\begin{restatable}{corollary}{firstcorbias}
\label{cor:bias}
As the time tends to infinity, we have that the empirical loss converges to zero, the Euclidean norm of the weights converges to infinity, and the weights converge in direction to some~$\vec{\theta}$ such that for all $j \in [k]$ we have $|a_j| = \|\vec{w}_j\|$, and if $a_j \neq 0$ then
\[\frac{\vec{w}_j}{\|\vec{w}_j\|} = \frac{\vec{v}_{\sgn(a_j)}}{\sum_{\sgn(a_{j'}) = \sgn(a_j)} a_{j'}^2}\;.\]
\end{restatable}

Thanks to homogeneity, the sign of the network output does not depend on the norm of the weights, only on their direction.  Examining networks whose weights are directional limits as in Corollary~\ref{cor:bias} is therefore informative of consequences for adversarial reprogramming of long training.  The following result tells us that, for any Bernoulli data model whose direction is in a half-space of the difference of the maximum-margin vectors, and for any adversarial program, the accuracy tends to~$1 / 2$ provided that the difficulty~$1 / \tau$ of the data model increases slower than the square root of the input dimension~$d$.  The latter assumption is considerably weaker than in the results of \citet{SchmidtSTTM18} on the Bernoulli data model, where the bound is in terms of the fourth root.  The statement also tells us that this failure cannot be fixed by choosing in advance a different mapping from the class labels of the original task to the class labels of the adversarial task.  Since we consider binary classification tasks here, the mapping can be represented by a multiplier $m \in \{\pm 1\}$.

\begin{restatable}{proposition}{firstprfail}
\label{pr:fail}
Suppose network weights~$\vec{\theta}$ are is in Corollary~\ref{cor:bias}, class label mapping $m \in \{\pm 1\}$ is arbitrary, data model~$\mathcal{D}$ is any $(\vec{\phi}, \rho, \tau)$-Bernoulli distribution such that $m \cos \angle(\vec{v}_1 - \vec{v}_{-1}, \vec{\phi}) < 0$, and adversarial program~$\vec{p}$ is arbitrary.  Then we have that
\[\mathbb{P}_{(\vec{x}, y) \sim \mathcal{D}}
  \{m \, y \, \mathcal{N}_{\vec{\theta}}(\vec{p} + \vec{x}) > 0\} \,\leq\,
  \frac{1}{2} +
  \frac{1}{2} \, e^{-2 d \tau^2 \cos^2 \angle(\vec{v}_1 - \vec{v}_{-1}, \vec{\phi})}\;.\]
\end{restatable}

\section{Experiments}
\label{s:exper}

\paragraph{Network architectures and initialisation.}
Our experiments\footnote{We are making code to run the experiments available at \url{https://github.com/englert-m/adversarial_reprogramming}.} are conducted using the following six network architectures:
ResNet-50~\citep{HeZRS16cvpr},
ResNet-50V2, ResNet-101V2, ResNet-152V2~\citep{HeZRS16eccv},
Inception-v3~\citep{SzegedyVISW16}, and
EfficientNet-B0~\citep{TanL19}.

We use the networks exactly as implemented in Keras in TensorFlow 2.8.1 including the method for randomly initialising the trainable weights. For biases, this means they are initialised with~$0$. For all other trainable weights, mostly, the Glorot uniform initialiser~\citep{GlorotB10} is used in this implementation. EfficientNet is an exception, where many layers instead use a truncated normal distribution that has mean~$0$ and standard deviation $\sqrt{2/\text{number of output units}}$.

All the networks we experiment with (ResNet-50, ResNet-50V2, ResNet-101V2, ResNet-152V2, Inception-v3, and EfficientNet-B0) involve batch normalisation layers~\citep{IoffeS15}. Every such layer maintains a moving mean and a moving variance based on batches it has seen during training. The inputs are then normalised accordingly: they are shifted by the recorded mean and scaled by the inverse of the recorded variance. Note that the moving mean and variance values are not trainable, i.e., they are not subject to updates by the optimiser during training. During inference, the moving mean and variance values are no longer updated, and the normalisation is performed based on the last values recorded during training.

Crucially, in the default implementation of these networks, these moving mean and moving variance values are initialised as~$0$ and~$1$, respectively. Therefore, an untrained network initialised in this way will behave as if the batch normalisation layers were not present.

To obtain more sensible random networks, i.e., ones that can still make use of batch normalisation, we initialise the moving mean and variance of batch normalisation layers differently. We generate a batch of 50 random images (each pixel value is chosen independently and uniformly at random in the allowed range). This single batch is then fed through the random network and each batch normalisation layer records the mean and variance values it sees at its input (and normalises its output accordingly). The trainable weights of the network are not changed during this process.

We should note that, in addition to the moving mean and variance, batch normalisation layers can have \emph{trainable} weights, by means of which the output mean and variance can be tuned. Specifically, such a layer may have trainable weights~$\gamma$ and~$\beta$, and will scale its otherwise normalised output by~$\gamma$ and shift it by~$\beta$. If present, these trainable weights are initialised as~$1$ and~$0$ respectively, and hence have no effect. Our initialisation procedure does not modify these trainable weights and they are therefore not used in our random networks.

For each network, after randomly initialising its weights as set out in Section~\ref{s:exper} and its batch normalisation layers as described above, we keep it completely fixed: neither its weights nor its batch normalisation layers (i.e., their moving means and variances, and their weights if any) change in any way.

\paragraph{Combining input images with adversarial programs.}
Our adversarial programs are colour images whose sizes match the expected input size of the networks. This is $224 \times 224$ for all networks except Inception-v3, where it is $299 \times 299$.

We use two different schemes to combine input images with the adversarial program. The first scheme, used by \citet{elsayed2018adversarial}, is to take the adversarial program and overwrite a portion of it by the input image. We do this in such a way that the input image is, up to rounding, centred in the adversarial program. We can vary the construction by scaling the input image up or down before applying this procedure. In particular, we use a parameter $r\in[0,1]$ and scale the input image, using bilinear interpolation, in such a way that $r$~times the width of the adversarial program is equal to the width of the scaled image. That we focus on the width is not important because all our inputs and programs are square.  An illustration is shown in Figure~\ref{f:scheme.1}.

\begin{figure}
\[\underbrace
  {\includegraphics[align=c,scale=.2]{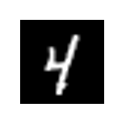}}
  _{\substack{\mathclap{\text{input from}} \\
              \mathclap{\text{adversarial task}}}}
  \quad \xrightarrow{\qquad} \quad
  \underbrace
  {\includegraphics[align=c,scale=.2]{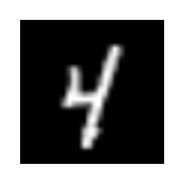}}
  _{\substack{\mathclap{\text{scaled input,}} \\
              \mathclap{\text{to width } r \cdot L}}}
  \quad + \quad
  \underbrace
  {\includegraphics[align=c,scale=.2]{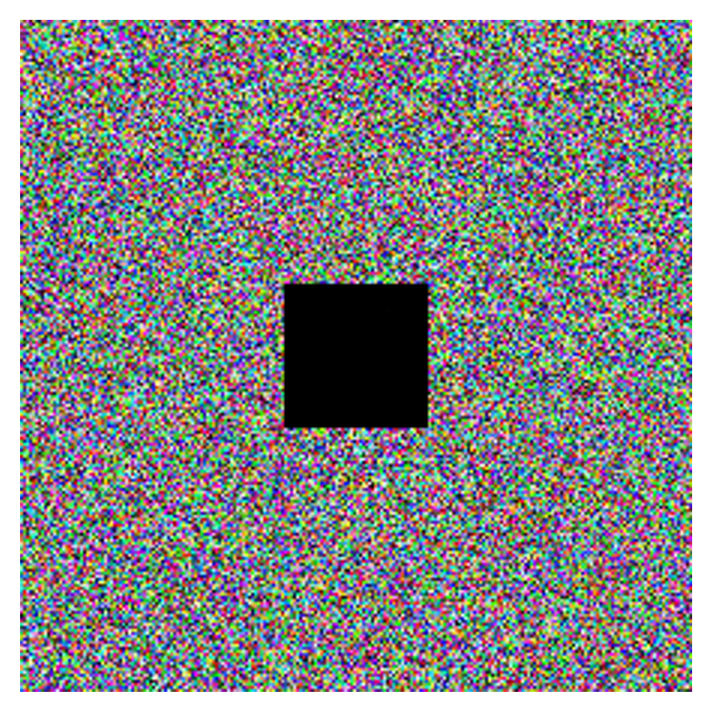}}
  _{\substack{\text{adversarial program,} \\
              \text{of width } L}}
  \quad = \quad
  \underbrace
  {\includegraphics[align=c,scale=.2]{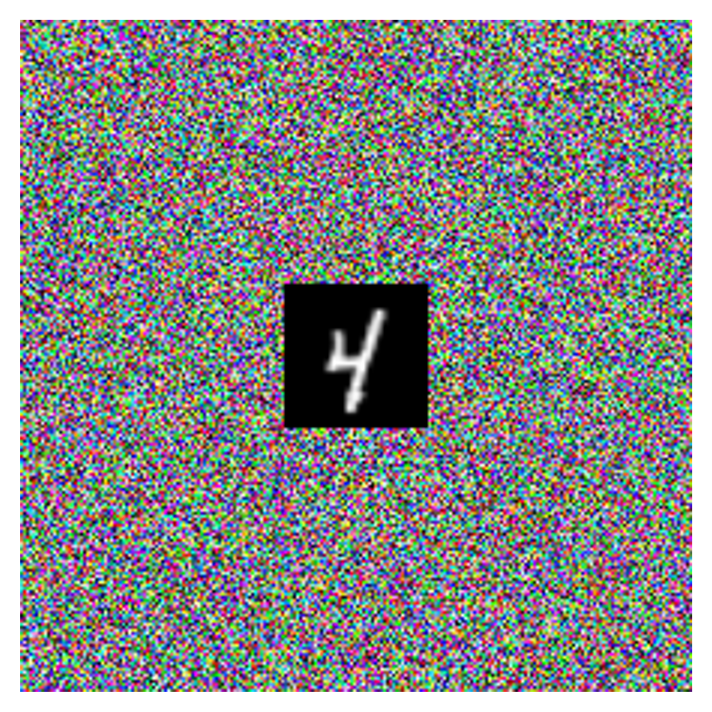}}
  _{\substack{\text{combined input,} \\
              \text{fed to network}}}\]
\caption{Scheme~1 for combining input images with adversarial programs.  In this example, the width and height of the adversarial program are~$224$ and the parameter~$r$ equals $2^{-20/9} \approx 0.214$, so the input image is scaled to width and height $r \cdot 224$~rounded, which is~$48$.}
\label{f:scheme.1}
\end{figure}

Our second scheme involves scaling the image to the same size as the adversarial program and then taking a convex combination of the two. We use a parameter $v\in[0,1]$ to specify how much weight the input image should get in this convex combination. Specifically, the combined image is obtained by calculating $v\cdot I + (1-v)\cdot P$, where $I$~is the input image and $P$~is the adversarial program.  An illustration is shown in Figure~\ref{f:scheme.2}.

\begin{figure}
\[\underbrace
  {\includegraphics[align=c,scale=.2]{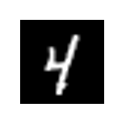}}
  _{\substack{\mathclap{\text{input from}} \\
              \mathclap{\text{adversarial task}}}}
  \; \xrightarrow{\qquad} \;
  v \cdot
  \underbrace
  {\includegraphics[align=c,scale=.2]{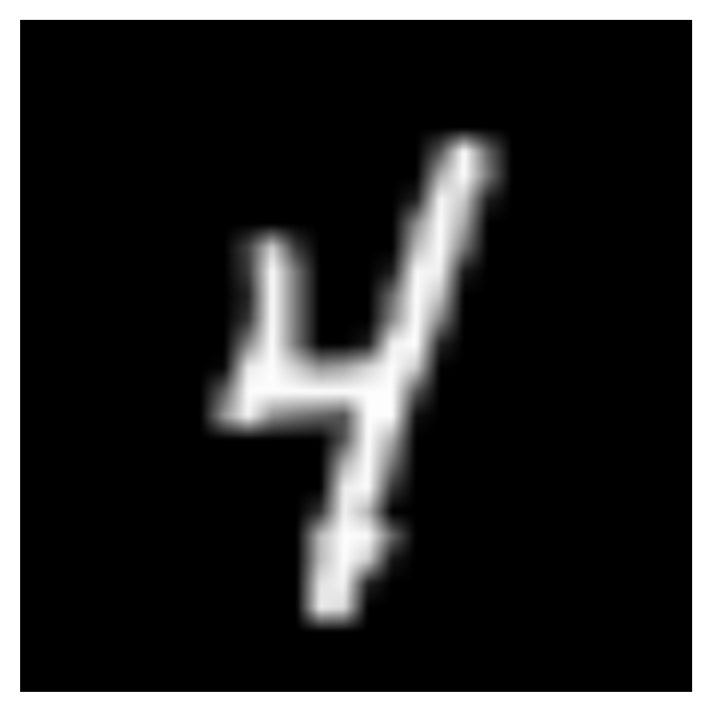}}
  _{\text{scaled input}}
  \; + \;
  (1 - v) \cdot
  \underbrace
  {\includegraphics[align=c,scale=.2]{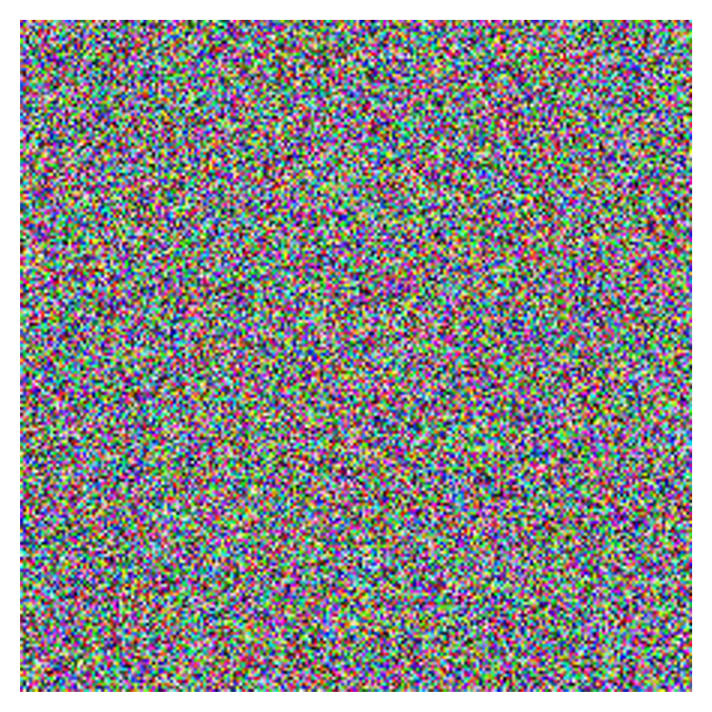}}
  _{\text{adversarial program}}
  \; = \;
  \underbrace
  {\includegraphics[align=c,scale=.2]{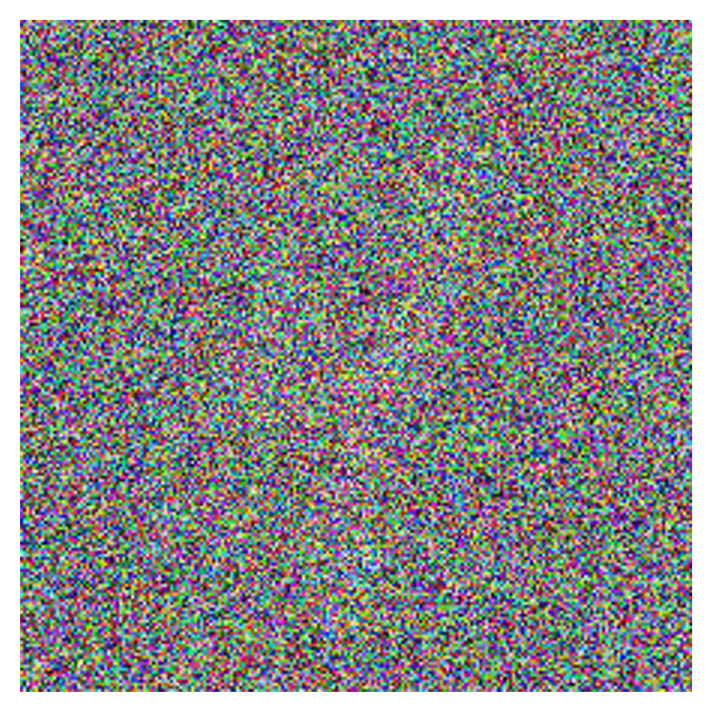}}
  _{\substack{\text{combined input,} \\
              \text{fed to network}}}\]
\caption{Scheme~2 for combining input images with adversarial programs.  In this example, the sizes of the scaled input image and the adversarial program are $224 \times 224$.  The parameter~$v$ equals $2^{-40/9} \approx 0.046$, so the weight of the input image in the convex combination with the adversarial program is approximately $4.6\%$, which makes it faintly visible.}
\label{f:scheme.2}
\end{figure}

\paragraph{Adversarial task dataset.}
\citet{elsayed2018adversarial} evaluate adversarial reprogramming on random networks using the MNIST~\citep{LeCunBBH98} dataset. In other words, they were asking whether it is possible to repurpose a random network for the task of classifying the handwritten digits from the MNIST dataset. We use the same dataset, which consists of $60,\!000$ training images and $10,\!000$ test images, for our experiments. It is available under the Creative Commons Attribution Share-Alike 3.0 licence.

The networks we use classify inputs into $1,\!000$ classes. We map the $10$~labels of the MNIST dataset onto the first~$10$ of these classes.

For additional experimental results on the Fashion-MNIST and Kuzushiji-MNIST datasets, please see \appendixref{app:other}.

\paragraph{Finding and evaluating adversarial programs.}
Internally, we represent adversarial programs using unconstrained weights. We then apply a softsign function to the weights to map them into the range $(-1,1)$, and further shift and scale the program such that the pixel values lie in the same range that is used for the input images. The program is initialised in such a way that after the application of the softsign function, each value lies uniformly at random in $(-1,1)$.

We use the $60,\!000$ training images to run an Adam optimiser~\citep{KingmaB14} with learning rate $0.01$ and a batch size of~$50$ to optimise the unconstrained weights of the adversarial program. We report the accuracy on the $10,\!000$ test images after $20$~epochs, please see Figure~\ref{f:experimentsmnist}.

The experiments were mainly run on two internal clusters utilising a mix of NVIDIA GPUs such as GeForce RTX 3080 Ti, Quadro RTX 6000, GeForce RTX 3060, GeForce RTX 2080 Ti, and GeForce GTX 1080.
Depending on the network, optimising a single adversarial program for 20 epochs takes between $30$~minutes and $1.5$~hours on a standard desktop computer with two NVIDIA GeForce RTX 3080 Ti GPUs.

We did not explore different optimisers and learning rates, since our first choices already resulted in suitable adversarial programs for these random networks. We only reduced the batch size to~$50$ after first trying~$100$, in order to reduce the requirement on GPU memory and be able to easily run the experiments on a wider range of hardware.
However, we did extensively explore the two different schemes of combing input images with adversarial programs, and different values for the respective parameters~$r$ and~$v$.
For each network, and each value of~$r$ and~$v$, we ran $5$~experiments, each with a new random initialisation of the network, and are reporting the average of the test accuracy.

\newcommand{\errorbarplot}[3]{%
\addplot
plot [error bars/.cd, y dir=none, y explicit]
table[x=#3, y=val-acc-mean-#1, y error plus expr=\thisrow{val-acc-max-#1}-\thisrow{val-acc-mean-#1}, y error minus expr=\thisrow{val-acc-mean-#1}-\thisrow{val-acc-min-#1}, col sep=comma]{aggregate_#2_#3.csv};
}

\begin{figure}
    \centering
    \begin{tikzpicture}
        \begin{groupplot}[group style={group size= 2 by 1},height=5.9cm,width=7.3cm]
            \nextgroupplot[xlabel=$r$, ylabel=test accuracy, xmode=log,minor tick num=1,ymin=0, ymax=1, xmax=1,log ticks with fixed point, ymajorgrids=true,  yminorgrids=true, minor grid style={line width=.2pt,draw=gray!50,legend pos=outer north east},title=Scheme 1]

            \errorbarplot{resnet50}{mnist}{r}
            \errorbarplot{ResNet50V2}{mnist}{r}
            \errorbarplot{ResNet101V2}{mnist}{r}
            \errorbarplot{inception-v3}{mnist}{r}
            \errorbarplot{EfficientNetB0}{mnist}{r}
            \errorbarplot{ResNet152V2}{mnist}{r}
            \coordinate (left) at (rel axis cs:0,1);

            \nextgroupplot[xlabel=$v$, xmode=log,minor tick num=1,ymin=0, ymax=1, xmax=1,log ticks with fixed point, ymajorgrids=true,  yminorgrids=true, minor grid style={line width=.2pt,draw=gray!50,legend pos=outer north east},title=Scheme 2,legend columns=3,legend to name=networksmnist]

\errorbarplot{resnet50}{mnist}{v}\addlegendentry{ResNet-50}
\errorbarplot{ResNet50V2}{mnist}{v}\addlegendentry{ResNet-50V2}
\errorbarplot{ResNet101V2}{mnist}{v}\addlegendentry{ResNet-101V2}
\errorbarplot{inception-v3}{mnist}{v}\addlegendentry{Inception-v3}
\errorbarplot{EfficientNetB0}{mnist}{v}\addlegendentry{EfficientNet-B0}
\errorbarplot{ResNet152V2}{mnist}{v}\addlegendentry{ResNet-152V2}

            \coordinate (right) at (rel axis cs:1,1);

        \end{groupplot}
        \coordinate (middle) at ($(left)!.5!(right)$);
        \node[below] at (middle |- current bounding box.south) {\pgfplotslegendfromname{networksmnist}};
    \end{tikzpicture}
    \caption{The accuracy achieved by the adversarial program on the MNIST test set for different parameters of the two schemes of combining input images with adversarial programs. The horizontal axes are logarithmic. The values plotted are averages over $5$~trials, which are listed together with the standard deviations in \appendixref{app:data}.}
    \label{f:experimentsmnist}
\end{figure}
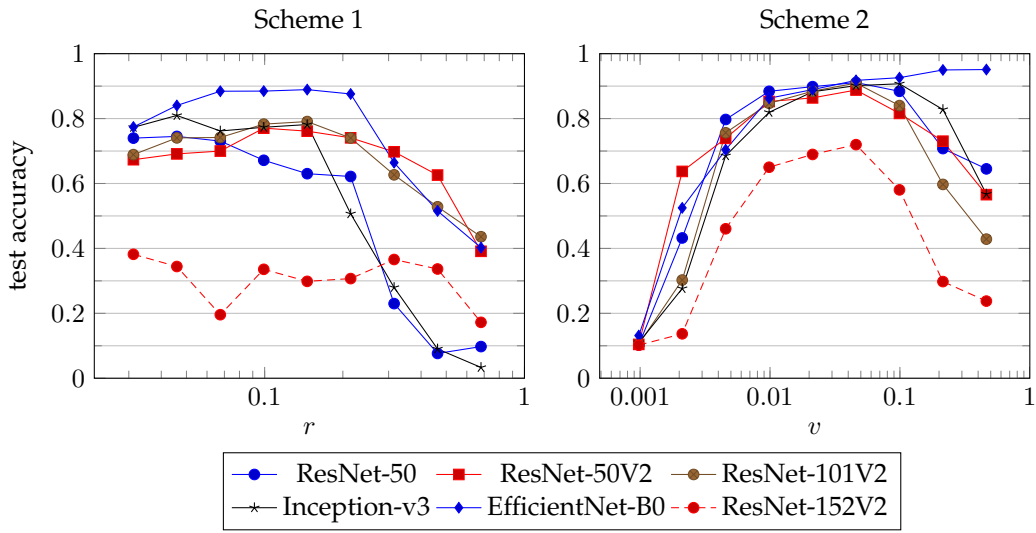

\paragraph{Discussion.}
Overall, the second scheme of combining input images with adversarial programs appears to give better and more reliable results in our experiments. For both schemes, the choice of parameters is important. Clearly, when~$r$ or~$v$ is~$0$, the input image is not visible to the network at all. On the other hand, when~$r$ or~$v$ is~$1$, there no longer is an adversarial program. In most cases, best results are achieved when the adversarial program is significantly larger (either by actual size in the first scheme, or in terms of pixel value ranges in the second scheme) than the input image.

In the second scheme in particular, we see accuracies on the test set which are lower, but not much lower than what \citet{elsayed2018adversarial} reported for networks trained on ImageNet. For $v\approx 0.046$ for example, we see accuracies of $91.8\%$, $91.1\%$, $90.9\%$, $90.2\%$, $88.8\%$, $72.0\%$ for EfficientNet-B0, ResNet-50, ResNet-101V2, Inception-v3, ResNet-50V2, ResNet-152V2, respectively.  This suggests that, while training, say, on ImageNet may impact the possibility of finding suitable adversarial programs, such a training may be less important than previously thought and other factors are of significant importance.

\section{Conclusion and future work}
\label{s:concl}

We proved the first theoretical results on adversarial reprogrammability of neural networks, in which we focused on architectures with two layers and ReLU activations, and on Bernoulli and Gaussian adversarial tasks.  Provided the input dimension is sufficiently large, and for a wide variety of parameter regimes, our results show that: firstly, arbitrarily high reprogramming accuracies are achievable in expectation for networks with random weights; and secondly, reprogramming accuracies that are no better than guessing may be unavoidable for networks that were trained for many iterations with small learning rates on orthogonally separable datasets.

In the theoretical results that conclude arbitrarily high expected reprogramming accuracies, we assumed that the width of the random network is no grater than its input dimension, which is similar to the assumption on widths in e.g.~\citet{DanielyS20} and enabled us to show existence of suitable adversarial programs by matrix inversion.  Interesting directions for future work include relaxing this assumption, extending the whole theoretical analysis to deeper networks, and investigating derivation of adversarial programs by gradient methods.

It would also be interesting to consider more permissive data models; however, our theoretical results on the failure of adversarial reprogramming on networks that were trained to infinity rely on implicit bias properties of gradient methods, and in that area separability assumptions on training data are common and appear challenging to lift (see e.g.~\citet{LyuLWA21}).

The outcomes of our experiments, which are on six realistic convolutional network architectures designed for image classification, and on three adversarial tasks provided by the MNIST, Fashion-MNIST and Kuzushiji-MNIST datasets, are consistent with our theoretical results on high reprogramming accuracies.  Both the experimental plots of test accuracies and the theoretical bounds of network outputs indicate existence of ``sweet spot'' maximising parameter values.\footnote{The varying of the scaling parameters~$r$ and~$v$ in the experiments corresponds in the Bernoulli data model case to varying the radius~$\rho$ while keeping all other parameters fixed, and in the Gaussian data model case to varying the mean radius~$\varrho$ and the variance parameter~$\varsigma$ while keeping fixed their ratio and all other parameters.}  Since in the experiments, the network architectures have widths, depths and other features that are currently beyond our theoretical assumptions, and the adversarial programs are derived by gradient descent, their outcomes provide further motivation for extending the theory as suggested above.

A conclusion that emerges across our experimental results is that the EfficientNet-B0 architecture tends to be more susceptible to adversarial reprogramming, and the ResNet-152V2 architecture tends to be less susceptible, than the remaining four.  We suggest for future work investigating the causes of this, as well as seeking to reprogram random networks for adversarial tasks that are more difficult than MNIST, Fashion-MNIST and Kuzushiji-MNIST.  In the context of the latter, it would be interesting to consider also vision transformer~\citep{DosovitskiyB0WZ21} architectures.

Another direction for experimental work is to verify that long training can cause adversarial reprogramming to fail.

\phantomsection
\addcontentsline{toc}{section}{Acknowledgments and Disclosure of Funding}
\ifneurips
\begin{ack}
\else
\section*{Acknowledgments and Disclosure of Funding}
\fi
We are grateful to the anonymous reviewers, whose comments helped us improve the paper.  We also thank Maria Ovens for linguistic advice.

This research was partially supported by the Centre for Discrete Mathematics and Its Applications (DIMAP) at the University of Warwick.
\ifneurips
\end{ack}
\else\fi

\phantomsection
\addcontentsline{toc}{section}{References}
\bibliographystyle{plainnat}
{\small \bibliography{main}}

\ifappendix
\ifchecklist\newpage\else\fi
\appendix

\tableofcontents

\section{Additional related works}
\label{app:related}

We focus on literature that is most related to this work and not already discussed in Section~\ref{s:intro}.

\paragraph{Understanding adversarial examples.}

\else\nocite{IlyasSTETM19}\nocite{ShamirMB21}\fi\ifappendix
Adversarial reprogramming can be seen as a challenging type of adversarial attack.  Instead of finding perturbations that push examples over a classification boundary, the goal is to find an adversarial program which is an offset that, when added to any input from the adversarial task, with high probability makes it classified as desired.  Equivalently, adversarial reprogramming seeks a single perturbation that pushes all inputs from the adversarial task to their respective target classes.  Therefore understanding when and why neural networks are susceptible to adversarial reprogramming is an interesting piece of the puzzle of understanding adversarial examples.  For example, our results on random networks suggest that sensitivity to well-generalising features in training data, which was identified by \citet{IlyasSTETM19} as a key cause of adversarial vulnerability, does not fully explain susceptibility to adversarial reprogramming.  A different conceptual framework for understanding adversarial examples, in terms of a dimpled manifold model, was proposed by \citet{ShamirMB21}; here our result that training for longer can be a defence against adversarial reprogramming suggests that it may be interesting to explore how its success depends on shapes of the dimples in classification boundaries as they evolve during training.

\else\nocite{BartlettBC21}\nocite{DanielyS20}\nocite{BubeckCGT21}\nocite{MontanariW22}\nocite{WangUMA22}\fi\ifappendix
More directly related to our work are the results of \citet{BartlettBC21} who, building on the works of \citet{DanielyS20} and \citet{BubeckCGT21}, proved that random ReLU networks of constant depth have small adversarial perturbations that can be found in one step along the direction opposite to the input gradient.  Those were advanced further recently by \citet{MontanariW22} removing a restriction on layer widths, and by \citet{WangUMA22} encompassing two-layer networks trained in the so-called lazy regime.  In comparison, our results show existence of adversarial programs that are points where the classification boundary behaves approximately as required for the adversarial task, and are sensitive to the radius and the difficulty of the latter as parameters; however we leave theoretical investigations of obtaining adversarial programs by gradient methods for future work.

\paragraph{Implicit bias of gradient descent.}

\else\nocite{SoudryHNGS18}\nocite{ZhangBHRV17}\fi\ifappendix
By showing that, for linear logistic regression on linearly separable data, gradient descent always converges to the maximum-margin solution, \citet{SoudryHNGS18} initiated a fruitful research direction on implicit bias of gradient descent, which tackles one of the greatest open questions in deep learning: why do overparameterised deep neural networks generalise well~\citep{ZhangBHRV17}.

\else\nocite{LyuL20}\nocite{JiT20}\nocite{PhuongL21}\nocite{WangP22}\nocite{LyuLWA21}\nocite{VardiYS22}\fi\ifappendix
Some of the works closest to ours are: \citet{LyuL20,JiT20}, who established that, for positively homogeneous deep networks and either exponential or logistic loss, if at some time training attains perfect accuracy and a small loss, then continuing the training makes the loss converge to~$0$ and the weights converge in direction to a Karush-Kuhn-Tucker point of a constrained optimisation problem on margin maximisation; \citet{PhuongL21}, who showed that, from small and balanced initialisations, and when trained with logistic loss on orthogonally separable data whose positive and negative support examples span the whole space, two-layer ReLU networks converge to a linear combination of two maximum-margin neurons; \cite{WangP22}, who derived a new proof via characterising the implicit bias of unregularised non-convex gradient flow as convex regularisation of an equivalent convex model; \citet{LyuLWA21}, who proved that, from small initialisations, and when trained with logistic loss on symmetric linearly separable data, two-layer networks with the leaky ReLU activation converge to a globally maximum-margin linear classifier, and that the result is fragile with respect to the symmetry assumption; and \citet{VardiYS22}, who showed that, for two-layer ReLU networks with first-layer biases and for data whose points are neither too many nor too correlated, if training converges then it produces non-robust solutions in spite of robust ones existing.  To supplement our discussion in Section~\ref{s:intro} of the motivations for considering random versus trained networks in this work, we remark that we see merit in the point of \citet{VardiYS22} that \emph{``trained networks are clearly not random, and properties that hold in random networks may not hold in trained networks.''}

\paragraph{Physical adversarial examples.}

\else\nocite{SharifBBR16}\nocite{BrownMRAG17}\fi\ifappendix
Currently more distant but with potential for interesting connections to our work is the vibrant research direction on adversarial examples in the physical world.  For example, our results may inform attempts to design adversarial accessories with prescribed effects for a set of participants~\citep{SharifBBR16}, or adversarial patches that cause prescribed classifications for a set of objects~\citep{BrownMRAG17}.

\section{Length of the adversarial program}
\label{app:p.p'}

To estimate the length of the adversarial program defined in Section~\ref{s:random}, we make use of the following known bounds on the smallest and largest singular values of matrices of independent identical Gaussians.

\else\nocite{VershyninEK12}\fi\ifappendix
\begin{theorem}[see e.g.\ {\citet[Corollary~5.35]{VershyninEK12}}]
\label{th:singular}
Let $\vec{A}$ be an $n \times m$ matrix whose entries are independent centred Gaussians with variance~$1 / d$.  Then with probability at least $1 - \gamma$ its smallest and largest singular values satisfy
\[\frac{\sqrt{m} - \sqrt{n} - \sqrt{2 \ln(2 / \gamma)}}{\sqrt{d}}
  \leq s_{\mathrm{min}}(\vec{A}) \leq s_{\mathrm{max}}(\vec{A}) \leq
  \frac{\sqrt{m} + \sqrt{n} + \sqrt{2 \ln(2 / \gamma)}}{\sqrt{d}}\;.\]
\end{theorem}

It follows that, in the regime $k = o(d)$, applying Theorem~\ref{th:singular} to the $k \times d$ matrix~$\vec{W}$ with $\gamma = o_d(1)$ and $\gamma = e^{-o(d)}$ gives us that with probability $1 - o_d(1)$ the singular values are within $1 \pm o_d(1)$.  Then, from the inequalities in~\eqref{eq:p.p'}, we have that $\|\vec{p}\|$ is close to~$\|\vec{p}'\|$ for large~$d$.  Observe that $\|\vec{p}'\| = \sqrt{d}$ whenever $K^-$ is not empty, which occurs with probability $1 - 2^{-k}$ because the events $\{j \in K^-\}$ are independent and have probability~$\nicefrac{1}{2}$.

\section{Proofs for random networks}
\label{app:random}

Here we work with the notations and assumptions from Section~\ref{s:random}, so we have: a random two-layer ReLU network~$\mathcal{N}$ with input dimension~$d$, width~$k$ such that $k \leq d$, and weights~$\vec{w}_j$ and~$a_j$ for $j \in [k]$; a $(\vec{\phi}, \rho, \tau)$-Bernoulli data model; and an adversarial program~$\vec{p}$ defined as in~\eqref{eq:p.p'}.

First we prove a lemma that provides a lower (respectively, upper) bound on the sum of outputs of the ``helpful'' neurons for the adversarial program~$\vec{p}$ together with inputs $\vec{x}$ that are positively (respectively, negatively) correlated with the direction~$\vec{\phi}$ of the data model.  In the proof, we use concentration properties of the weights of~$\mathcal{N}$ to show that these neurons get close to computing together the inner product $\vec{\phi}^\top \vec{x}$, which equals $\rho \cos \alpha$ where $\alpha$~is the angle between~$\vec{\phi}$ and~$\vec{x}$.  The probability parameters~$\gamma$ and~$\gamma'$ can be regarded as constants that can be chosen to be arbitrarily small.  If $1 / |\cos \alpha|$ grows slower than the square root of the width~$k$, then the bound, up to a constant factor, is essentially $\sqrt{k / d} \, \rho \cos \alpha$.

\begin{lemma}
\label{l:K+}
Suppose $\vec{x} \in \mathbb{R}^d$ is such that $\|\vec{x}\| = \rho$ and $\vec{\phi}^\top \vec{x} \neq 0$.  Let $\alpha \coloneqq \angle(\vec{\phi}, \vec{x})$ and $y \coloneqq \sgn(\cos \alpha)$.  Then with probability at least $1 - \gamma - \gamma' / \sqrt{2 \pi}$ one has
\[y \!\sum_{j \in K^+}\! a_j \psi(\vec{w}_j^\top (\vec{p} + \vec{x})) \,>\,
  \frac{\sqrt{k} \rho}{\sqrt{d}} \!
  \left(\!\frac{\sqrt{\pi}}{4} |\cos \alpha| \!
          \left(\frac{1}{4 \sqrt{2}} -
                \sqrt{\frac{\ln(1 / \gamma)}{k}}\right) -
          \sqrt{\frac{2 \ln(1 / \gamma')}{k}}\right)\;,\]
provided that $\gamma' \leq 1 / \sqrt{e}$.
\end{lemma}

\begin{proof}
Let us consider the case $y = 1$, i.e.~$\vec{\phi}^\top \vec{x} = \rho \cos \alpha > 0$.

We reason as follows, where $\stackrel{\mathrm{d}}{=}$ denotes distributional equivalence:
\begin{align*}
& \sum_{j \in K^+} a_j \psi(\vec{w}_j^\top (\vec{p} + \vec{x})) \\
& = \sum_{j \in K^+} a_j \psi(\vec{w}_j^\top \vec{x})
& & \text{since $j \in K^+$ implies $\vec{w}_j^\top \vec{p} = 0$} \\
& = \sum_{j \in K^+} \frac{\sgn(\vec{w}_j^\top \vec{\phi})}{\sqrt{k}}
                     \psi(\vec{w}_j^\top \vec{x})
& & \text{by definition of~$K^+$} \\
& \stackrel{\mathrm{d}}{=}
  \sum_{j \in [k]} B_j \frac{\sgn(\vec{w}_j^\top \vec{\phi})}{\sqrt{k}}
                       \psi(\vec{w}_j^\top \vec{x})
& & \text{where $B_j$ are independent $\mathrm{Ber}(\nicefrac{1}{2})$} \\
& = \sum_{\substack{B_j = 1 \\
          \mathclap{\vec{w}_j^\top \vec{\phi} > 0} \\
          \mathclap{\vec{w}_j^\top \vec{x} > 0}}}
    \frac{\vec{w}_j^\top \vec{x}}{\sqrt{k}}
  - \sum_{\substack{B_{j'} = 1 \\
          \mathclap{\vec{w}_{j'}^\top \vec{\phi} < 0} \\
          \mathclap{\vec{w}_{j'}^\top \vec{x} > 0}}}
    \frac{\vec{w}_{j'}^\top \vec{x}}{\sqrt{k}}
& & \text{by definition of~$\psi$} \\
& = \sum_{\substack{B_j = 1 \\
          \mathclap{\vec{w}_j^\top \vec{\phi} > 0} \\
          \mathclap{\vec{w}_j^\top \vec{x} > 0}}}
    \frac{\vec{w}_j^\top \vec{x}}{\sqrt{k}}
  + \sum_{\substack{B_{j'} = 1 \\
          \mathclap{-\vec{w}_{j'}^\top \vec{\phi} > 0} \\
          \mathclap{-\vec{w}_{j'}^\top \vec{x} < 0}}}
    \left(-\frac{\vec{w}_{j'}^\top \vec{x}}{\sqrt{k}}\right)
& & \text{and observe $j$ and $j'$ are disjoint} \\
& \stackrel{\mathrm{d}}{=}
  \sum_{\substack{B_j = 1 \\
        \mathclap{\vec{w}_j^\top \vec{\phi} > 0} \\
        \mathclap{\vec{w}_j^\top \vec{x} > 0}}}
  \frac{\vec{w}_j^\top \vec{x}}{\sqrt{k}}
+ \sum_{\substack{B_{j'} = 1 \\
        \mathclap{\vec{w}_{j'}^\top \vec{\phi} > 0} \\
        \mathclap{\vec{w}_{j'}^\top \vec{x} < 0}}}
  \frac{\vec{w}_{j'}^\top \vec{x}}{\sqrt{k}}
& & \text{by independence and symmetry of~$\vec{w}_{j'}$} \\
& = \sum_{\substack{B_j = 1 \\
          \mathclap{\vec{w}_j^\top \vec{\phi} > 0}}}
    \frac{\vec{w}_j^\top \vec{x}}{\sqrt{k}}
& & \text{by merging the sums} \\
& = \sum_{\substack{B_j = 1 \\
          \mathclap{\vec{w}_j^\top \vec{\phi} > 0}}}
    \frac{\vec{w}_j^\top \vec{\phi} \vec{\phi}^\top \vec{x}}{\sqrt{k}} +
    \sum_{\substack{B_j = 1 \\
          \mathclap{\vec{w}_j^\top \vec{\phi} > 0}}}
    \frac{\vec{w}_j^\top (\vec{I}_d - \vec{\phi} \vec{\phi}^\top) \vec{x}}{\sqrt{k}}
& & \text{projecting onto~$\vec{\phi}$ and orthogonal hyperplane} \\
& \stackrel{\mathrm{d}}{=}
  \underbrace
  {\sum_{B'_j = 1}
   \frac{|\vec{w}_j^\top \vec{\phi}| \vec{\phi}^\top \vec{x}}{\sqrt{k}}}
  _{\text{\ref{en:I}}} +
  \underbrace
  {\sum_{B'_j = 1}
   \frac{\vec{w}_j^\top (\vec{I}_d - \vec{\phi} \vec{\phi}^\top) \vec{x}}{\sqrt{k}}}
  _{\text{\ref{en:II}}}
& & \text{where $B'_j$ are independent $\mathrm{Ber}(\nicefrac{1}{4})$.}
\end{align*}

We now analyse the two sums separately, in both cases obtaining a lower bound:
\begin{enumerate}[(I),left=0pt]
\item
\label{en:I}
Each of the~$k$ terms is the absolute value of an independent centred Gaussian with variance $(\rho \cos \alpha)^2 / (d k)$, so it is with probability~$\nicefrac{1}{2}$ greater than
\[\frac{\sqrt{2} \erf^{-1}(1 / 2) \rho \cos \alpha}{\sqrt{d k}} >
  \frac{\sqrt{\pi} \rho \cos \alpha}{2 \sqrt{2 d k}}\;.\]
\else\nocite{vershynin2018high}\fi\ifappendix
Also each of the~$k$ terms is present in the sum independently with probability~$\nicefrac{1}{4}$.  Hence by Hoeffding's inequality (see e.g.\ \citet[Theorem~2.2.6]{vershynin2018high}), we have that with probability at least $1 - \gamma$ the sum is greater than
\begin{equation}
  \frac{\sqrt{\pi} \rho \cos \alpha}{2 \sqrt{2 d k}}
  \left(\frac{k}{8} -
        \sqrt{\frac{k \ln(1 / \gamma)}{2}}\right)
= \frac{\sqrt{\pi k} \rho \cos \alpha}{4 \sqrt{d}}
  \left(\frac{1}{4 \sqrt{2}} -
        \sqrt{\frac{\ln(1 / \gamma)}{k}}\right)\;.
\label{eq:I}
\end{equation}
\item
\label{en:II}
Each term is an independent centred Gaussian with variance $(\rho \sin \alpha)^2 / (d k)$, so a sum of~$k'$ such terms is a centred Gaussian with variance $(\rho \sin \alpha)^2 k' / (d k)$, which is with probability at least $1 - \gamma' / \sqrt{2 \pi}$ at least
\begin{equation}
-\rho \, |\sin \alpha| \sqrt{\frac{2 k' \ln(1 / \gamma')}{d k}} \geq
-\rho \sqrt{\frac{2 \ln(1 / \gamma')}{d}}\;,
\quad\text{provided}\quad
\gamma' \leq 1 / \sqrt{e}\;.
\label{eq:II}
\end{equation}
\end{enumerate}

Summing the bounds in~\eqref{eq:I} and~\eqref{eq:II} yields the bound in the statement, and we can combine the probabilities by the union bound.

The case $y = -1$, i.e.~$\vec{\phi}^\top \vec{x} = \rho \cos \alpha < 0$, is analogous.
\end{proof}

Second we prove a lemma that gives us a bound on how much the sum of outputs of the``unhelpful'' neurons can spoil the work of the ``helpful'' neurons which the previous lemma addressed.  The function~$f$ is the density of a standard Gaussian, i.e.~$f(u) = e^{-u^2 / 2} / \sqrt{2 \pi}$.  It bounds the first summand inside the outer parentheses, which is therefore exponentially small for large $d / (\sqrt{k} \rho)$.  Again viewing the probability parameter~$\gamma''$ as arbitrarily small but constant, the second summand approaches zero at the rate of the square root of the width~$k$.  The techniques we use in the proof are similar to those for the previous one.

\begin{lemma}
\label{l:K-}
Suppose $\vec{x} \in \mathbb{R}^d$ is such that $\|\vec{x}\| = \rho$ and $\vec{\phi}^\top \vec{x} \neq 0$.  Let $\alpha \coloneqq \angle(\vec{\phi}, \vec{x})$ and $y \coloneqq \sgn(\cos \alpha)$.  Then with probability at least $1 - \gamma''$ one has
\[y \!\sum_{j \in K^-}\! a_j \psi(\vec{w}_j^\top (\vec{p} + \vec{x})) \,>\,
  -\frac{\sqrt{k} \rho}{\sqrt{d}} \!
  \left(\!f\!\left(\frac{d}{\sqrt{k} \rho}\right)
          \min\!\left\{1, {\left(\frac{\sqrt{k} \rho}{d}\right)\!}^2\right\} +
          \frac{2 \pi}{\pi - 1} \sqrt{\frac{\ln(1 / \gamma'')}{k}}\right)\;.\]
\end{lemma}

\begin{proof}
Let us consider the case $y = 1$, i.e.~$\vec{\phi}^\top \vec{x} = \rho \cos \alpha > 0$.

We estimate as follows:
\begin{align*}
& \sum_{j \in K^-} a_j \psi(\vec{w}_j^\top (\vec{p} + \vec{x})) \\
& = \sum_{j \in K^-} a_j \psi(\vec{w}_j^\top \vec{x} - \sqrt{d / |K^-|})
& & \text{since $j \in K^-$ implies $\vec{w}_j^\top \vec{p} = -\sqrt{d / |K^-|}$} \\
& \geq -\sum_{j \in K^+} \frac{\psi(\vec{w}_j^\top \vec{x} - \sqrt{d / |K^-|})}{\sqrt{k}}
& & \text{since $a_j \geq -1 / \sqrt{k}$} \\
& \geq -\sum_{j \in K^+} \frac{\psi(\vec{w}_j^\top \vec{x} - U)}{\sqrt{k}}
& & \text{where $U \coloneqq \sqrt{d / k}$} \\
& \geq -\sum_{j \in [k]} \frac{\psi(\vec{w}_j^\top \vec{x} - U)}{\sqrt{k}}
& & \text{since $\psi(u) \geq 0$ for all $u$.}
\end{align*}

Each $\vec{w}_j^\top \vec{x}$ is a centred Gaussian with variance $\rho^2 / d \eqqcolon \sigma^2$.  Hence (see e.g.\ \citet[Proposition~2.1.2]{vershynin2018high})
\[\mathbb{P}\{\vec{w}_j^\top \vec{x} \geq U\} \geq
  \max\{\sigma / U - (\sigma / U)^3, 0\} \cdot f(U / \sigma)\;.\]
We therefore have
\begin{multline*}
\mu \coloneqq
\mathbb{E}[\psi(\vec{w}_j^\top \vec{x} - U)] =
\int_{U}^\infty (u - U) \frac{f(u / \sigma)}{\sigma} \, \mathrm{d} u \\ =
\sigma f(U / \sigma) - U \, \mathbb{P}\{\vec{w}_j^\top \vec{x} \geq U\} \leq
\sigma f(U / \sigma) \min\{1, (\sigma / U)^2\}\;.
\end{multline*}
Writing~$X_j$ for the centred random variable $\psi(\vec{w}_j^\top \vec{x} - U) - \mu$, there are the following two cases.
\paragraph{Case $0 \leq s \leq \mu$.}
\begin{multline*}
\mathbb{P}\{|\vec{w}_j^\top \vec{x}| \geq s\} \geq
\mathbb{P}\{|\vec{w}_j^\top \vec{x}| \geq \mu\} =
1 - \int_{-\mu}^\mu \frac{f(u / \sigma)}{\sigma} \, \mathrm{d} u >
1 - 2 \mu f(0) / \sigma \\ \geq
1 - 2 f(0) f(U / \sigma) >
1 - \frac{1}{\pi} \geq
\left(1 - \frac{1}{\pi}\right) \mathbb{P}\{|X_j| \geq s\}\;.
\end{multline*}
\paragraph{Case $\mu < s$.}
\[\mathbb{P}\{|\vec{w}_j^\top \vec{x}| \geq s\} \geq
  2 \, \mathbb{P}\{\vec{w}_j^\top \vec{x} \geq U + \mu + s\} =
  2 \, \mathbb{P}\{\psi(\vec{w}_j^\top \vec{x} - U) \geq \mu + s\} =
  2 \, \mathbb{P}\{|X_j| \geq s\}\;.\]
\else\nocite{wainwright2019high}\fi\ifappendix
Hence~$X_j$ is sub-Gaussian with parameter $\sqrt{2} \pi \sigma / (\pi - 1)$
(see e.g.\ \citet[proof of Theorem~2.6]{wainwright2019high}), so by Hoeffding's bound (see e.g.\ \citet[Proposition~2.5]{wainwright2019high}), with probability at least $1 - \gamma''$ one has
\begin{multline*}
\sum_{j \in K^-} a_j \psi(\vec{w}_j^\top (\vec{p} + \vec{x})) \geq
-\sum_{j \in [k]} \frac{\mu + X_j}{\sqrt{k}} =
-\sqrt{k} \mu -\sum_{j \in [k]} \frac{X_j}{\sqrt{k}} \\ >
-\sqrt{k} \mu -\frac{2 \pi}{\pi - 1} \sigma \sqrt{\ln(1 / \gamma'')} \geq
-\sqrt{k} \, \sigma f(U / \sigma) \min\{1, (\sigma / U)^2\}
-\frac{2 \pi}{\pi - 1} \sigma \sqrt{\ln(1 / \gamma'')} \\ =
-\frac{\sqrt{k} \rho}{\sqrt{d}} f\!\left(\frac{d}{\sqrt{k} \rho}\right)
 \min\!\left\{1, {\left(\frac{\sqrt{k} \rho}{d}\right)\!}^2\right\}
-\frac{2 \pi}{\pi - 1} \rho \sqrt{\frac{\ln(1 / \gamma'')}{d}}\;,
\end{multline*}
which equals the bound in the statement.

The case $y = -1$, i.e.~$\vec{\phi}^\top \vec{x} = \rho \cos \alpha < 0$, is again analogous.
\end{proof}

Our main technical result on random networks provides a lower (respectively, upper) bound on the output of the adversarially reprogrammed network for positively (respectively, negatively) labelled inputs sampled from the Bernoulli data model.

\begin{theorem}
\label{th:random}
There exist positive constants $C_1$, $C_2$, $C_3$, $C_4$ and~$C_5$ such that, with probability at least $(1 - C_1 \gamma) (1 - \gamma^\dag)$ over the draw of the network~$\mathcal{N}$ and the labelled data point~$(\vec{x}, y)$ from the $(\vec{\phi}, \rho, \tau)$-Bernoulli distribution, provided that $2 d \tau^2 \geq \ln(1 / \gamma^\dag)$, we have
\[y \, \mathcal{N}(\vec{p} + \vec{x}) >
  \frac{\sqrt{k} \rho}{\sqrt{d}} \!
  \left(\!
  C_2 \tau -
  C_3 \exp\!\left(\!-\frac{d^2}{2 k \rho^2}\!\right)
      \min\!\left\{\!1, \frac{k \rho^2}{d^2}\!\right\} -
  C_4 \sqrt{\frac{\ln(1 / \gamma)}{k}} -
  C_5 \sqrt{\frac{\ln(1 / \gamma^\dag)}{d}}
  \right)\;.\]
\end{theorem}

\begin{proof}
Let $\alpha \coloneqq \angle(\vec{\phi}, \vec{x})$.

\else\nocite{SchmidtSTTM18}\fi\ifappendix
Since $(\vec{x}, y)$ is sampled from the $(\vec{\phi}, \rho, \tau)$-Bernoulli distribution, \citet[Lemma~24]{SchmidtSTTM18} tell us that, for all $u \in (0, 2 \tau]$,
\[\mathbb{P}\{y \cos \alpha > 2 \tau - u\} \geq
  1 - e^{-d u^2 / 2}\;,\]
i.e.~that with probability at least $1 - \gamma^\dag$ one has
\[y \cos \alpha >
  2 \tau - \sqrt{\frac{2 \ln(1 / \gamma^\dag)}{d}}\;,\]
provided that $2 d \tau^2 \geq \ln(1 / \gamma^\dag)$.

It remains to apply Lemmas~\ref{l:K+} and~\ref{l:K-} with $\gamma = \gamma' = \gamma''$ and $C_1 = 2 + 1 / \sqrt{2 \pi}$, observe that $1 - C_1 \gamma \geq 0$ implies $\gamma \leq 1 / \sqrt{e}$, and simplify as follows:
\begin{align*}
&
\frac{\sqrt{\pi}}{4} y \cos \alpha \!
\left(\frac{1}{4 \sqrt{2}} -
      \sqrt{\frac{\ln(1 / \gamma)}{k}}\right) -
\sqrt{\frac{2 \ln(1 / \gamma)}{k}} \\ & -
f\!\left(\frac{d}{\sqrt{k} \rho}\right)
\min\!\left\{1, {\left(\frac{\sqrt{k} \rho}{d}\right)\!}^2\right\} -
\frac{2 \pi}{\pi - 1} \sqrt{\frac{\ln(1 / \gamma)}{k}} \\
& \geq
\frac{\sqrt{\pi}}{16 \sqrt{2}} y \cos \alpha -
f\!\left(\frac{d}{\sqrt{k} \rho}\right)
\min\!\left\{1, {\left(\frac{\sqrt{k} \rho}{d}\right)\!}^2\right\} \\ & \quad\,\, -
\left(\!\sqrt{2} + \frac{\sqrt{\pi}}{4} + \frac{2 \pi}{\pi - 1}\!\right)
\sqrt{\frac{\ln(1 / \gamma)}{k}} \\
& >
\frac{\sqrt{\pi}}{8 \sqrt{2}} \tau -
f\!\left(\frac{d}{\sqrt{k} \rho}\right)
\min\!\left\{1, {\left(\frac{\sqrt{k} \rho}{d}\right)\!}^2\right\} \\ & \quad\,\, -
\left(\!\sqrt{2} + \frac{\sqrt{\pi}}{4} + \frac{2 \pi}{\pi - 1}\!\right)
\sqrt{\frac{\ln(1 / \gamma)}{k}} -
\frac{\sqrt{\pi}}{16}
\sqrt{\frac{\ln(1 / \gamma^\dag)}{d}}\;,
\end{align*}
which establishes the statement with
\[C_2 = \frac{\sqrt{\pi}}{8 \sqrt{2}}\;, \quad
  C_3 = 1 / \sqrt{2 \pi}\;, \quad
  C_4 = \sqrt{2} + \frac{\sqrt{\pi}}{4} + \frac{2 \pi}{\pi - 1} \quad\text{and}\quad
  C_5 = \frac{\sqrt{\pi}}{16}\;. \mbox{\qedhere}\]
\end{proof}

We now restate and prove Corollary~\ref{cor:random} from Section~\ref{s:random}.

\firstcorrandom*

\begin{proof}
Fix~$\gamma$ and~$\gamma^\dag$ such that $(1 - C_1 \gamma) (1 - \gamma^\dag)$ is as close to~$100\%$ as required.

By the assumptions in the statement, we have $\eta_{(\tau)} < 1 / 2$, so $\tau = \Omega(d^{-\eta_{(\tau)}}) = \omega(d^{-1 / 2})$ and hence $2 d \tau^2 \geq \ln(1 / \gamma^\dag)$ for large enough~$d$; moreover, we have
\[\exp\!\left(\!\frac{d^2}{k \rho^2}\!\right) = \omega((1 / \tau)^2)\;, \quad
  \frac{k}{\ln(1 / \gamma)} = \omega((1 / \tau)^2) \quad\text{and}\quad
  \frac{d}{\ln(1 / \gamma^\dag)} = \omega((1 / \tau)^2)\;.\]

The corollary follows by applying Theorem~\ref{th:random} and observing that
\begin{multline*}
C_2 \tau -
C_3 \exp\!\left(\!-\frac{d^2}{2 k \rho^2}\!\right)
    \min\!\left\{\!1, \frac{k \rho^2}{d^2}\!\right\} -
C_4 \sqrt{\frac{\ln(1 / \gamma)}{k}} -
C_5 \sqrt{\frac{\ln(1 / \gamma^\dag)}{d}} \\
=
\frac{1}{O(1 / \tau)} -
\frac{1}{\omega(1 / \tau) \ln(\omega((1 / \tau)^2))} -
\frac{1}{\omega(1 / \tau)}
=
\frac{1}{O(1 / \tau)}\;,
\end{multline*}
which is positive for large enough~$d$.
\end{proof}

\section{Clarke subdifferential}
\label{app:Clarke}

\else\nocite{clarke1975generalized}\fi\ifappendix
By Rademacher's theorem, every locally Lipschitz function~$g: \mathbb{R}^m \to \mathbb{R}$ is differentiable almost everywhere.  Its Clarke subdifferential~\citep{clarke1975generalized} $\partial g$ at a point~$\vec{z}$ is the convex closure of the set of all limits of gradients along sequences that converge to~$\vec{z}$:
\[\partial g(\vec{z}) \coloneqq
  \mathrm{conv} \left\{\lim_{i \to \infty} \nabla g(\vec{z}_i) \,\mid\,
                       g \text{ differentiable at all } \vec{z}_i,
                       \text{ and } \lim_{i \to \infty} \vec{z}_i =
                                    \vec{z}\right\}\;,\]
which is nonempty and compact.  If $g$~is continuously differentiable at~$\vec{z}$, then $\partial g(\vec{z}) = \{\nabla g(\vec{z})\}$.

For example, the Clarke subdifferential of the ReLU function is:
\[\partial \psi(u) =
  \begin{cases}
  \{0\}  & \text{if $u < 0$,} \\
  [0, 1] & \text{if $u = 0$,} \\
  \{1\}  & \text{if $u > 0$.}
  \end{cases}\]

\else\nocite{LyuL20}\fi\ifappendix
For a fuller introduction to the Clarke subdifferential and its applications to gradient flow, and for further references to related literature, we refer the reader to e.g.~\citet[Section~3 and Appendix~I]{LyuL20}.

\section{Karush-Kuhn-Tucker conditions}
\label{app:KKT}

\else\nocite{DuttaDTA13}\fi\ifappendix
Following \citet[Section~2.2]{DuttaDTA13}, supposing $g, h_1, \ldots, h_n: \mathbb{R}^m \to \mathbb{R}$ are locally Lipschitz, we have that $\vec{z} \in \mathbb{R}^m$ is a Karush-Kuhn-Tucker point of the single-objective constrained optimisation problem
\[\text{\rm minimise}\quad
  g(\vec{z})
  \quad\text{\rm subject to}\quad
  \forall i \in [n]: \, h_i(\vec{z}) \leq 0\]
if and only if there exist Lagrange multipliers $\lambda_1, \ldots, \lambda_n \geq 0$ such that:
\begin{description}[labelwidth=\widthof{\bf (complementary slackness)},align=parright]
\item[(feasibility)]
$\forall i \in [n]: \, h_i(\vec{z}) \leq 0$,
\item[(equilibrium inclusion)]
$\vec{0} \in \partial g(\vec{z}) + \sum_{i = 1}^n \lambda_i \partial h_i(\vec{z})$, and
\item[(complementary slackness)]
$\forall i \in [n]: \, \lambda_i h_i(\vec{z}) = 0$.
\end{description}

\section{Proofs for implicit bias}
\label{app:bias}

Here we work with the notations and assumptions from Section~\ref{s:bias}, so we have a trajectory $\vec{\theta}(t): [0, \infty) \to \mathbb{R}^{k (d + 1)}$ of gradient flow for a two-layer ReLU network~$\mathcal{N}$ with input dimension~$d$ and width~$k$, trained on an orthogonally separable binary classification dataset $S = \{(\vec{x}_1, y_1), \ldots, (\vec{x}_n, y_n)\}$, from a balanced and live initialisation, using either the exponential $\ell_\mathrm{exp}(u) = e^{-u}$ or the logistic $\ell_\mathrm{log}(u) = \ln(1 + e^{-u})$ loss function.

For $t \in [0, \infty)$, we denote by~$\mathcal{N}_{\vec{\theta}(t)}$ the network at time~$t$ of gradient flow, i.e.~whose weights are the coordinates of the trajectory vector at time~$t$: $\vec{w}_1(t)$, \ldots, $\vec{w}_k(t)$ for the first layer, and $a_1(t)$, \ldots, $a_k(t)$ for the second layer.

\else\nocite{DavisDKL20}\fi\ifappendix
Since our network functions, loss functions, and their compositions, are definable (see \citet[Corollary~5.11]{DavisDKL20}), they admit the chain rule (see \citet[Theorem~5.8]{DavisDKL20}), and hence we have the following basic fact on the empirical loss $\mathcal{L}(\vec{\theta}) = \sum_{i = 1}^n \ell(y_i \mathcal{N}_{\vec{\theta}}(\vec{x}_i))$ decreasing along gradient flow.

\begin{proposition}[by {\citet[Lemma~5.2]{DavisDKL20}}]
\label{pr:loss}
We have $\mathrm{d} \mathcal{L}(\vec{\theta}) / \mathrm{d} t = -\|\mathrm{d} \vec{\theta} / \mathrm{d} t\|^2$ for almost all $t \in [0, \infty)$.
\end{proposition}

From the differential inclusion of gradient flow, i.e.~that $\mathrm{d} \vec{\theta} / \mathrm{d} t \in -\partial \mathcal{L}(\vec{\theta}(t))$ for almost all $t \in [0, \infty)$, it is straightforward to obtain the following expressions for the derivatives of the weights during training.  Note that for the first layer we have a membership rather than an equation because the Clarke subdifferentials of the ReLU function may be at zero.

\begin{proposition}
\label{pr:dw.da}
For all $j \in [k]$ and almost all $t \in [0, \infty)$ we have
\begin{align*}
\frac{\mathrm{d} \vec{w}_j}{\mathrm{d} t}
& \in -\sum_{i = 1}^n
       \ell'(y_i \mathcal{N}_{\vec{\theta}}(\vec{x}_i))
       y_i a_j \partial \psi(\vec{w}_j^\top \vec{x}_i) \vec{x}_i \\
\frac{\mathrm{d} a_j}{\mathrm{d} t}
& = -\sum_{i = 1}^n
     \ell'(y_i \mathcal{N}_{\vec{\theta}}(\vec{x}_i))
     y_i \psi(\vec{w}_j^\top \vec{x}_i)\;.
\end{align*}
\end{proposition}

\else\nocite{PhuongL21}\fi\ifappendix
We now extend to exponential loss the observation made for logistic loss by \citet[Lemma~A.4]{PhuongL21}: that during training second-layer weights remain balanced and keep their signs.

\begin{lemma}
\label{l:balanced}
Throughout the training, for all $j \in [k]$, we have that $|a_j| = \|\vec{w}_j\|$ and that $a_j$ maintains its sign.
\end{lemma}

\begin{proof}
Observe that the ReLU function satisfies $\partial \psi(u) u = \{\psi(u)\}$ for all $u \in \mathbb{R}$.  Observe also that $|\ell'(u)| \leq \ell(u)$ for all $u \in \mathbb{R}$, which for exponential loss is trivial, and for logistic loss follows from $e < (1 + 1 / e^u)^{e^u + 1}$, i.e.\ $1 < \ell_\mathrm{log}(u) / |\ell_\mathrm{log}'(u)|$.

From Proposition~\ref{pr:dw.da} it follows that $\mathrm{d} \|\vec{w}_j\|^2 / \mathrm{d} t = 2 \vec{w}_j^\top (\mathrm{d} \vec{w}_j / \mathrm{d} t)$ and $\mathrm{d} a_j^2 / \mathrm{d} t = 2 a_j (\mathrm{d} a_j / \mathrm{d} t)$, and so $\mathrm{d} \|\vec{w}_j\|^2 / \mathrm{d} t = \mathrm{d} a_j^2 / \mathrm{d} t$ almost everywhere during the training.

Hence, by our assumption of balanced initialisation, we have that $|a_j| = \|\vec{w}_j\|$ throughout the training.  Then, by Proposition~\ref{pr:loss}, for almost all $t \in [0, \infty)$ we have that $a_j(t) \neq 0$ implies
\[\left|\frac{\mathrm{d} \ln a_j^2(t)}{\mathrm{d} t}\right|
  \,\leq\, 2 \sum_{i = 1}^n |\ell'(y_i \mathcal{N}_{\vec{\theta}(t)}(\vec{x}_i))| \, \|\vec{x}_i\|
  \,\leq\, 2 \mathcal{L}(\vec{\theta}(t)) \max_{i = 1}^n \|\vec{x}_i\|
  \,\leq\, 2 \mathcal{L}(\vec{\theta}(0)) \max_{i = 1}^n \|\vec{x}_i\|\;.\]
Therefore $\ln a_j^2$ is bounded below by a linear function of~$t$, so $a_j$ does not cross zero throughout the training.
\end{proof}

We now observe several key properties that the weights have during the training in our setting.  Namely: the offsets of the positive neurons remain in the cone spanned by the positive inputs and the opposites of the negative inputs, their inner products with all the latter vectors are nondecreasing, and their norms either remain bounded or tend to infinity; and the analogous holds for the negative neurons.  For $\vec{u}_1, \ldots, \vec{u}_m \,\in\, \mathbb{R}^l$, let $\mathrm{cone}(\vec{u}_1, \ldots, \vec{u}_m) \coloneqq \{\sum_{i = 1}^m b_i \vec{u}_i \,\mid\, b_1, \ldots, b_m \in [0, \infty)\}$.  Also let $\vec{X} \coloneqq \{y_i \vec{x}_i \,\mid\, i \in [n]\}$.

\begin{lemma}
\label{l:key}
\begin{enumerate}[(i),left=0pt]
\item
\label{en:i}
For all $j \in [k]$ and all $t \in [0, \infty)$, we have $\vec{w}_j(t) - \vec{w}_j(0) \in \mathrm{cone}(\sgn(a_j) \vec{X})$.
\item
\label{en:ii}
For all $i \in [n]$, all $j \in [k]$, and almost all $t \in [0, \infty)$, we have
$\mathrm{d} (\vec{w}_j^\top \sgn(a_j) y_i \vec{x}_i) / \mathrm{d} t \geq 0$.
\item
\label{en:iii}
For each $j \in [k]$, either $|a_j|$ is bounded, or $|a_j| \to \infty$ as $t \to \infty$.
\end{enumerate}
\end{lemma}

\begin{proof}
We have~\ref{en:i} by Proposition~\ref{pr:dw.da} and the negativity of the derivatives of both the exponential and the logistic loss functions.

For~\ref{en:ii}, we have by Proposition~\ref{pr:dw.da} that
\[\frac{\mathrm{d} (\vec{w}_j^\top \sgn(a_j) y_i \vec{x}_i)}{\mathrm{d} t}
  \in -\sum_{i' = 1}^n
       \ell'(y_{i'} \mathcal{N}_{\vec{\theta}}(\vec{x}_{i'}))
       |a_j| \partial \psi(\vec{w}_j^\top \vec{x}_{i'})
       y_i y_{i'} \vec{x}_i^\top \vec{x}_{i'}\;,\]
which never contains negative values since the dataset is orthogonally separable.

By Lemma~\ref{l:balanced}, it suffices to show~\ref{en:iii} for~$\|\vec{w}_j\|$, and so letting $\widehat{\vec{w}}_j(t) \coloneqq \vec{w}_j(t) - \vec{w}_j(0)$, it suffices to show~\ref{en:iii} for~$\|\widehat{\vec{w}}_j\|$.

Suppose $\widehat{\vec{w}}_j^\top\!(t) \sgn(a_j(t)) y_i \vec{x}_i$ is bounded for all $i \in [n]$.  By orthogonal separability of the dataset, every two vectors in $\mathrm{cone}(\sgn(a_j) \vec{X})$ have nonnegative inner product, and so by~\ref{en:i} we have that
\[\|\widehat{\vec{w}}_j(t)\| \leq
  \sum_{i \in [n]}
  \frac{\widehat{\vec{w}}_j^\top\!(t) \sgn(a_j(t)) y_i \vec{x}_i}{\|\vec{x}_i\|}\;.\]
Hence $\|\widehat{\vec{w}}_j(t)\|$ is also bounded.

Otherwise, by~\ref{en:ii}, there exists $i \in [n]$ such that $\widehat{\vec{w}}_j^\top\!(t) \sgn(a_j(t)) y_i \vec{x}_i \to \infty$ as $t \to \infty$.  By orthogonal separability of the dataset, every two vectors in $\mathrm{cone}(\sgn(a_j) \vec{X})$ have nonnegative inner product, and so by~\ref{en:i} we have that
\[\frac{\widehat{\vec{w}}_j^\top\!(t) \sgn(a_j(t)) y_i \vec{x}_i}{\|\vec{x}_i\|}
  \leq \|\widehat{\vec{w}}_j(t)\|\;.\]
Hence also $\|\widehat{\vec{w}}_j(t)\| \to \infty$ as $t \to \infty$.
\end{proof}

The following is our main lemma, whose statement is simple: for every input, there must be at least one ReLU whose output tends to infinity during the training.  For its proof, recall from Section~\ref{s:bias} the notation $I_s = \{i \in [n] \,\mid\, y_i = s\}$.

\begin{lemma}
\label{l:main}
For all $i \in [n]$, there exists $j \in [k]$ such that $\vec{w}_j^\top \vec{x}_i \to \infty$ as $t \to \infty$.
\end{lemma}

\begin{proof}
Consider the case $i \in I_1$.

Recalling that $\vec{w}_{j_1}^\top \vec{x}_{i_1}$ is positive at $t = 0$ by our assumption of live initialisation, and that it is nondecreasing during the training by Lemma~\ref{l:key}~\ref{en:ii}, we have by Lemma~\ref{l:balanced} that $a_{j_1}$~is bounded below by a positive constant.  Since $\mathrm{d} (\vec{w}_{j_1}^\top \vec{x}_{i_1}) / \mathrm{d} t \geq -\ell'(\mathcal{N}_{\vec{\theta}}(\vec{x}_{i_1})) a_{j_1} \|\vec{x}_{i_1}\|^2$ for almost all $t \in [0, \infty)$, we have that either $\mathcal{N}_{\vec{\theta}}(\vec{x}_{i_1})$ or $\vec{w}_{j_1}^\top \vec{x}_{i_1}$ is not bounded above.  In either case, by Lemma~\ref{l:key}~\ref{en:iii}, there exists $j \in [k]$ such that $a_j \to \infty$ as $t \to \infty$.

Suppose $-\vec{w}_j^\top \vec{x}_{i'}$ is unbounded for some $i' \in I_{-1}$.  Recalling that, by Lemma~\ref{l:key}~\ref{en:i}, $\vec{w}_j(t) - \vec{w}_j(0) \in \mathrm{cone}(\vec{X})$ for all $t \in [0, \infty)$, it follows from orthogonal separability of the dataset that $-\vec{w}_j^\top \vec{x}_{i'}$ is unbounded for all $i' \in I_{-1}$.  Hence there exists $T \in [0, \infty)$ such that, for all $i' \in I_{-1}$ and $t \in [T, \infty)$, we have $\vec{w}_j^\top \vec{x}_{i'} < 0$.  Then $\vec{w}_j(t) - \vec{w}_j(T) \in \mathrm{cone}\{\vec{x}_i \,\mid\, i \in I_1\}$ for all $t \in [T, \infty)$, so $\vec{w}_j^\top \vec{x}_{i'}$ is unbounded for some $i' \in I_1$.

Therefore $\vec{w}_j^\top \vec{x}_{i'}$ is unbounded for some $i' \in I_1$.  Arguing as before, since $\vec{w}_j(t) - \vec{w}_j(0) \in \mathrm{cone}(\vec{X})$ for all $t \in [0, \infty)$, we have that $\vec{w}_j^\top \vec{x}_{i'}$ is unbounded for all $i' \in I_1$, in particular for $i' = i$ as required.

The proof for the case $i \in I_{-1}$ is analogous.
\end{proof}

That brings us to a position where we can restate and prove Theorem~\ref{th:loss} from Section~\ref{s:bias}.  Using Lemma~\ref{l:main}, we show that, whereas the empirical loss is bounded below by zero, eventually its rate of decrease remains at least the sum of squares of derivatives of the loss function, which implies that the latter and hence also the empirical loss get arbitrarily small.

\firstthloss*

\begin{proof}
By Lemma~\ref{l:main}, for each $i \in [n]$, let $j(i) \in [k]$ be such that $\vec{w}_{j(i)}^\top \vec{x}_i \to \infty$ as $t \to \infty$.  Then let $T \in [0, \infty)$ be such that, for all $i \in [n]$ and all $t \in [T, \infty)$, we have $|a_{j(i)}| \geq 1 / \|\vec{x}_i\|$ and $\vec{w}_{j(i)}^\top \vec{x}_i > 0$.

Recalling Propositions~\ref{pr:loss} and \ref{pr:dw.da}, for almost all $t \in [T, \infty)$ we have
\begin{align*}
\frac{\mathrm{d} \mathcal{L}(\vec{\theta})}{\mathrm{d} t}
& = -\left\|\frac{\mathrm{d} \vec{\theta}}{\mathrm{d} t}\right\|^2
  \leq -\sum_{j = 1}^k
        \left\|\frac{\mathrm{d} \vec{w}_j}{\mathrm{d} t}\right\|^2 \\
& \in  -\sum_{j = 1}^k
        \left\|\sum_{i = 1}^n
               \ell'(y_i \mathcal{N}_{\vec{\theta}}(\vec{x}_i))
               y_i a_j \partial \psi(\vec{w}_j^\top \vec{x}_i) \vec{x}_i\right\|^2 \\
& \subseteq -\sum_{j = 1}^k
            {\left(\sum_{i = 1}^n
                   \ell'(y_i \mathcal{N}_{\vec{\theta}}(\vec{x}_i))
                   y_i a_j \partial \psi(\vec{w}_j^\top \vec{x}_i) \vec{x}_i\!\right)\!\!}^\top \!\!\!
             \left(\sum_{i' = 1}^n
                   \ell'(y_{i'} \mathcal{N}_{\vec{\theta}}(\vec{x}_{i'}))
                   y_{i'} a_j \partial \psi(\vec{w}_j^\top \vec{x}_{i'}) \vec{x}_{i'}\!\right) \\
& = -\sum_{j = 1}^k \sum_{i = 1}^n \sum_{i' = 1}^n
     \ell'(y_i \mathcal{N}_{\vec{\theta}}(\vec{x}_i))
     \ell'(y_{i'} \mathcal{N}_{\vec{\theta}}(\vec{x}_{i'}))
     a_j^2 \partial \psi(\vec{w}_j^\top \vec{x}_i) \partial \psi(\vec{w}_j^\top \vec{x}_{i'})
     y_i y_{i'} \vec{x}_i^\top \vec{x}_{i'} \\
& \leq \Bigg\{\!-\min \sum_{i = 1}^n \sum_{j = 1}^k
                      \left(\ell'(y_i \mathcal{N}_{\vec{\theta}}(\vec{x}_i))
                            a_j \partial \psi(\vec{w}_j^\top \vec{x}_i)
                            \|\vec{x}_i\|\right)^2\!\Bigg\} \\
& \leq \Bigg\{\!-\min \sum_{i = 1}^n
                      \left(\ell'(y_i \mathcal{N}_{\vec{\theta}}(\vec{x}_i))
                            a_{j(i)} \partial \psi(\vec{w}_{j(i)}^\top \vec{x}_i)
                            \|\vec{x}_i\|\right)^2\!\Bigg\} \\
& \leq \Bigg\{\!-\sum_{i = 1}^n
                 (\ell'(y_i \mathcal{N}_{\vec{\theta}}(\vec{x}_i)))^2\!\Bigg\}\;,
\end{align*}
where the second inequality is by orthogonal separability of the dataset.

Hence $\mathcal{L}(\vec{\theta}(T)) \,\geq\, \int_T^\infty \sum_{i = 1}^n (\ell'(y_i \mathcal{N}_{\vec{\theta}}(\vec{x}_i)))^2 \, \mathrm{d} t$, so for all $\nu > 0$ there exists $t \in [T, \infty)$ such that $\sum_{i = 1}^n (\ell'(y_i \mathcal{N}_{\vec{\theta}}(\vec{x}_i)))^2 < \nu$.

For exponential loss, let $t_0$ be such that $\sum_{i = 1}^n (\ell_\mathrm{exp}'(y_i \mathcal{N}_{\vec{\theta}(t_0)}(\vec{x}_i)))^2 < 1 / n^2$.  Then for all $i \in [n]$ we have $y_i \mathcal{N}_{\vec{\theta}(t_0)}(\vec{x}_i) > \ln n$, so $\mathcal{L}(\vec{\theta}(t_0)) < 1 = \ell_\mathrm{exp}(0)$.

For logistic loss, let $t_0$ be such that $\sum_{i = 1}^n (\ell_\mathrm{log}'(y_i \mathcal{N}_{\vec{\theta}(t_0)}(\vec{x}_i)))^2 < 1 / (2 n + 1)^2$.  Then for all $i \in [n]$ we have $y_i \mathcal{N}_{\vec{\theta}(t_0)}(\vec{x}_i) > \ln(2 n)$ and thus $\ell_\mathrm{log}(y_i \mathcal{N}_{\vec{\theta}(t_0)}(\vec{x}_i)) < \ln(1 + 1 / (2 n)) < 1 / (2 n)$, so $\mathcal{L}(\vec{\theta}(t_0)) < 1 / 2 < \ln 2 = \ell_\mathrm{log}(0)$.
\end{proof}

\else\nocite{JiT20}\fi\ifappendix
The following theorem from recent literature is on the late phase of gradient flow, i.e.~after attaining perfect accuracy and a small loss.  Its assumptions on the network are satisfied in our setting (see \citet[Section~2]{JiT20}), in particular two-layer ReLU networks are positively $2$-homogeneous, i.e.~for all $\alpha > 0$, $\vec{\theta}$ and~$\vec{x}$, we have $\mathcal{N}_{\alpha \vec{\theta}}(\vec{x}) = \alpha^2 \mathcal{N}_{\vec{\theta}}(\vec{x})$.  For KKT conditions for nonsmooth optimisation problems, see Appendix~\ref{app:KKT}.

\begin{theorem}[{\citet[Theorem~A.8]{LyuL20} and \citet[Theorem~3.1]{JiT20}}]
\label{th:KKT}
Suppose $\mathcal{N}_{\vec{\theta}}$~is a neural network which is locally Lipschitz, positively homogeneous, and definable in some \mbox{o-minimal} structure that includes the exponential function.  Consider minimising either the exponential or the logistic loss over a binary classification dataset $(\vec{x}_1, y_1), \ldots, (\vec{x}_n, y_n)$ using gradient flow.  If $\mathcal{L}(\vec{\theta}(0)) < \ell(0)$, then as $t \to \infty$ we have that $\mathcal{L}(\vec{\theta}(t)) \to 0$, $\|\vec{\theta}(t)\| \to \infty$, and the weights converge in direction to a Karush–Kuhn–Tucker point of the following maximum margin problem:
\[\text{\rm minimise}\quad
  \frac{1}{2} \|\vec{\theta}\|^2
  \quad\text{\rm subject to}\quad
  \forall i \in [n]: \, y_i \mathcal{N}_{\vec{\theta}}(\vec{x}_i) \geq 1\;.\]
\end{theorem}

\else\nocite{LyuLWA21}\fi\ifappendix
Another recent result is the next lemma, which shows that maximum-margin KKT points for two-layer ReLU networks on orthogonally separable datasets are particularly streamlined.  Recall from Section~\ref{s:bias} that~$\vec{v}_1$ and~$\vec{v}_{-1}$ denote the maximum margin vectors for the positive and negative data classes, respectively.  We remark that it may be interesting that a part of the statement says that, regardless of whether the initialisation was balanced, the weights of any KKT point are balanced.  We provide a proof of the lemma, which is perhaps simpler, and also it removes a marginal concern about applicability to our setting due to the leaky ReLU activation having been assumed in \citet[Lemma~B.9]{LyuLWA21}.

\begin{lemma}[{\citet[Lemma~H.3]{LyuLWA21}}]
\label{l:KKT}
Suppose $\vec{\theta}$~is a Karush–Kuhn–Tucker point of the problem in Theorem~\ref{th:KKT} for a two-layer ReLU network~$\mathcal{N}_{\vec{\theta}}$ and an orthogonally separable dataset.  Then for all $j \in [k]$ we have $|a_j| = \|\vec{w}_j\|$, and if $a_j \neq 0$ then
\[\frac{\vec{w}_j}{\|\vec{w}_j\|} = \frac{\vec{v}_{\sgn(a_j)}}{\sum_{\sgn(a_{j'}) = \sgn(a_j)} a_{j'}^2}\;.\]
\end{lemma}

\begin{proof}
We have that $\vec{\theta}$~is feasible, i.e.~$\forall i \in [n]: \, y_i \mathcal{N}_{\vec{\theta}}(\vec{x}_i) \geq 1$, and some $\lambda_1, \ldots, \lambda_n \geq 0$ satisfy
\[\vec{\theta} \in \sum_{i = 1}^n \lambda_i \partial (y_i \mathcal{N}_{\vec{\theta}}(\vec{x}_i))
  \quad\text{and}\quad
  \forall i \in [n]: \, \lambda_i = 0 \,\vee\,
                        y_i \mathcal{N}_{\vec{\theta}}(\vec{x}_i) = 1\;.\]
Thus, for all $j \in [k]$, we have $\vec{w}_j = a_j \sum_{i = 1}^n \lambda_i y_i \psi'_{i, j} \vec{x}_i$ where $\psi'_{i, j} \in \partial \psi(\vec{w}_j^\top \vec{x}_i)$ are some values, and $a_j = \sum_{i = 1}^n \lambda_i y_i \psi(\vec{w}_j^\top \vec{x}_i)$.

That $a_j = 0$ implies $\|\vec{w}_j\| = 0$ is immediate, so suppose $a_j > 0$.

For all $i' \in I_{-1}$ such that $\lambda_{i'} > 0$, we have by orthogonal separability that
\[\vec{w}_j^\top \vec{x}_{i'}
  =
  a_j \sum_{i = 1}^n \lambda_i y_i \psi'_{i, j} \vec{x}_i^\top \vec{x}_{i'}
  \leq
  -a_j \lambda_{i'} \psi'_{i', j} \|\vec{x}_{i'}\|^2\;,\]
so $\psi'_{i', j} > 0$ is impossible, and thus $\psi'_{i', j} = 0$.

Hence $a_j = \sum_{i \in I_1} \lambda_i \psi(\vec{w}_j^\top \vec{x}_i)$, so there exists $i \in I_1$ such that $\lambda_i > 0$ and $\vec{w}_j^\top \vec{x}_i > 0$.  Then for all $i' \in I_1$ we have $\vec{w}_j^\top \vec{x}_{i'} \geq a_j \lambda_i \vec{x}_i^\top \vec{x}_{i'} > 0$.

Therefore $a_j^2 = \vec{w}_j^\top a_j \sum_{i \in I_1} \lambda_i \vec{x}_i = \|\vec{w}_j\|^2$, so $a_j = \|\vec{w}_j\|$ and $\vec{w}_j / \|\vec{w}_j\| = \sum_{i \in I_1} \lambda_i \vec{x}_i \eqqcolon \vec{u}_1$.

Analogously, if $a_j < 0$, then $-a_j = \|\vec{w}_j\|$ and $\vec{w}_j / \|\vec{w}_j\| = \sum_{i \in I_{-1}} \lambda_i \vec{x}_i \eqqcolon \vec{u}_{-1}$.

For both signs $s \in \{\pm 1\}$, let $\kappa_s \coloneqq \sum_{\sgn(a_j) = s} a_j^2$.  By orthogonal separability, it follows that for all $i \in I_s$ we have $\mathcal{N}_{\vec{\theta}}(\vec{x}_i) = \sum_{\sgn(a_j) = s} a_j \vec{w}_j^\top \vec{x}_i = \kappa_s \vec{u}_s^\top \vec{x}_i$.  Hence $\kappa_s \vec{u}_s$~is a Karush-Kuhn-Tucker point of the maximum margin problem
\[\text{\rm minimise}\quad
  \frac{1}{2} \|\vec{v}\|^2
  \quad\text{\rm subject to}\quad
  \forall i \in I_s: \, \vec{v}^\top \vec{x}_i \geq 1\;,\]
and therefore its unique Karush-Kuhn-Tucker point, which is the maximum margin vector~$\vec{v}_s$.
\end{proof}

The following corollary, restated from Section~\ref{s:bias}, is now an immediate consequence of Theorem~\ref{th:loss}, Theorem~\ref{th:KKT} and Lemma~\ref{l:KKT}.

\firstcorbias*

It remains to prove Proposition~\ref{pr:fail}, also restated from Section~\ref{s:bias}.

\firstprfail*

\begin{proof}
The cases $m = 1$ and $m = -1$ are symmetric by swapping $\vec{\phi}$~and~$y$ with $-\vec{\phi}$~and~$-y$, so we can assume that $m = 1$.

Since
\begin{multline*}
\mathbb{P}_{(\vec{x}, y) \sim \mathcal{D}}
\{y \, \mathcal{N}_{\vec{\theta}}(\vec{p} + \vec{x}) > 0\} \\
=
\frac{1}{2} \,
\mathbb{P}_{(\vec{x}, y) \sim \mathcal{D}}
\{y \, \mathcal{N}_{\vec{\theta}}(\vec{p} + \vec{x}) > 0 \,\mid\, y = 1\} +
\frac{1}{2} \,
\mathbb{P}_{(\vec{x}, y) \sim \mathcal{D}}
\{y \, \mathcal{N}_{\vec{\theta}}(\vec{p} + \vec{x}) > 0 \,\mid\, y = -1\}\;,
\end{multline*}
it suffices to establish the following claim: there exists a data class $s \in \{\pm 1\}$ such that
\[\mathbb{P}_{(\vec{x}, y) \sim \mathcal{D}}
  \{y \, \mathcal{N}_{\vec{\theta}}(\vec{p} + \vec{x}) > 0 \,\mid\, y = s\} \,\leq\,
  e^{-2 d \tau^2 \cos^2 \angle(\vec{v}_1 - \vec{v}_{-1}, \vec{\phi})}\;.\]

Suppose $\vec{v}_1^\top \vec{p} \leq \vec{v}_{-1}^\top \vec{p}$, and let $s = 1$.  Then we have that
$\psi(\vec{v}_1^\top    (\vec{p} + \vec{x})) >
 \psi(\vec{v}_{-1}^\top (\vec{p} + \vec{x}))$
implies
$\vec{v}_1^\top \vec{x} >
 \vec{v}_{-1}^\top \vec{x}$,
so
\begin{multline*}
\mathbb{P}_{(\vec{x}, y) \sim \mathcal{D}}
\{y \, \mathcal{N}_{\vec{\theta}}(\vec{p} + \vec{x}) > 0 \,\mid\, y = s\}
=
\mathbb{P}_{(\vec{x}, 1) \sim \mathcal{D}}
\{\mathcal{N}_{\vec{\theta}}(\vec{p} + \vec{x}) > 0\} \\
=
\mathbb{P}_{(\vec{x}, 1) \sim \mathcal{D}}\!
\left\{\psi(\vec{v}_1^\top    (\vec{p} + \vec{x})) >
       \psi(\vec{v}_{-1}^\top (\vec{p} + \vec{x}))\right\}
\leq
\mathbb{P}_{(\vec{x}, 1) \sim \mathcal{D}}\!
\left\{\vec{v}_1^\top \vec{x} >
       \vec{v}_{-1}^\top \vec{x}\right\} \\
\leq
e^{-2 d \tau^2 \cos^2 \angle(\vec{v}_1 - \vec{v}_{-1}, \vec{\phi})}
\end{multline*}
as required, where the last inequality is by the assumption $\cos \angle(\vec{v}_1 - \vec{v}_{-1}, \vec{\phi}) < 0$ and \citet[Lemma~26]{SchmidtSTTM18}.

If $\vec{v}_1^\top \vec{p} \geq \vec{v}_{-1}^\top \vec{p}$, the argument to show the claim is analogous, with taking $s = -1$ and observing that
$\psi(\vec{v}_1^\top    (\vec{p} + \vec{x})) <
 \psi(\vec{v}_{-1}^\top (\vec{p} + \vec{x}))$
implies
$\vec{v}_1^\top \vec{x} <
 \vec{v}_{-1}^\top \vec{x}$.
\end{proof}

\else\nocite{GilmerMFSRWG18}\nocite{SchmidtSTTM18}\nocite{IlyasSTETM19}\fi\ifappendix
\section{Gaussian data models}
\label{app:Gaussian}

We show here how our main theoretical results can be adapted if the Bernoulli data models (see Section~\ref{s:random}) we considered are replaced by data models based on the multivariate Gaussian distribution.  The latter distribution has been standard for modelling data in theoretical investigations of adversarial examples (see e.g.\ \citet{GilmerMFSRWG18,SchmidtSTTM18,IlyasSTETM19}).

\paragraph{Definition.}

Following \citet{SchmidtSTTM18}, given a per-class mean vector $\varrho \vec{\varphi}$ where $\varrho > 0$ is its radius and $\vec{\varphi} \in \mathbb{S}^{d - 1}$ is its direction, and given a variance parameter $\varsigma > 0$, we define the $(\vec{\varphi}, \varrho, \varsigma)$-Gaussian distribution over $(\vec{x}, y) \,\in\, \mathbb{R}^d \times \{\pm 1\}$ as follows:
\begin{itemize}
\item
first draw the label~$y$ uniformly at random from $\{\pm 1\}$,
\item
then sample the data point~$\vec{x}$ from the spherical multivariate Gaussian distribution whose mean is $y \varrho \vec{\varphi}$ and covariance matrix is $\varsigma^2 \vec{I}_d$.
\end{itemize}

These binary classification data models are thus mixtures of two spherical Gaussians, one per data class.  The overlapping of the classes, and therefore the difficulty of the classification task, is controlled by the ratio between the spherical radius~$\sqrt{d} \varsigma$ and the mean radius~$\varrho$, i.e.\ classification becomes harder as $\sqrt{d} \varsigma / \varrho$ increases.

\subsection{Random networks}

We assume that, as in Section~\ref{s:random}, we have a random two-layer ReLU network~$\mathcal{N}$ with input dimension~$d$, width~$k$ such that $k \leq d$, and weights~$\vec{w}_j$ and~$a_j$ for $j \in [k]$.

Let $\vec{p}$~be an adversarial program defined as in~\eqref{eq:p.p'}, with respect to the direction~$\vec{\varphi}$ of the Gaussian data model.

By adapting the proof of Theorem~\ref{th:random} to the Gaussian adversarial task, we obtain the following result, which provides similar lower (respectively, upper) bounds on the output of the adversarially reprogrammed network for positively (respectively, negatively) labelled adversarial inputs.

\begin{theorem}
\label{th:G.random}
There exist positive constants $C_1$, $C_2$, $C_3$ and~$C_4$ such that, with probability at least $(1 - C_1 \gamma) (1 - 2 \gamma^\dag)$ over the draw of the network~$\mathcal{N}$ and the labelled data point~$(\vec{x}, y)$ from the $(\vec{\varphi}, \varrho, \varsigma)$-Gaussian distribution, provided that $\varrho^2 / \varsigma^2 \geq 2 \ln(1 / \gamma^\dag)$, we have
\[y \, \mathcal{N}(\vec{p} + \vec{x}) >
  \frac{\sqrt{k}}{\sqrt{d}} \!
  \left(\!
  \frac{C_2}{2} \beta -
  \beta' \!
  \left(\!
  C_3 \exp\!\left(\!-\frac{d^2}{2 k {\beta'}^2}\!\right)
      \min\!\left\{\!1, \frac{k {\beta'}^2}{d^2}\!\right\} +
  C_4 \sqrt{\frac{\ln(1 / \gamma)}{k}}
  \right)\!
  \right)\;,\]
where
\[\beta  \coloneqq \varrho - \varsigma \sqrt{2 \ln(1 / \gamma^\dag)}
  \quad\text{and}\quad
  \beta' \coloneqq \varrho + \varsigma (\sqrt{d} + \sqrt{2 \ln(1 / \gamma^\dag)})\;.\]
\end{theorem}

\begin{proof}
For any fixed $(\vec{x}, y)$ such that $y \vec{\varphi}^\top \vec{x} > 0$, we can apply Lemmas~\ref{l:K+} and~\ref{l:K-} as in the proof of Theorem~\ref{th:random}, which gives us that with probability at least $1 - C_1 \gamma$ over the draw of the network~$\mathcal{N}$ one has
\[y \, \mathcal{N}(\vec{p} + \vec{x}) >
  \frac{\sqrt{k}}{\sqrt{d}} \!
  \left(\!
  \frac{C_2}{2} y \vec{\varphi}^\top \vec{x} -
  \|\vec{x}\| \!
  \left(\!
  C_3 \exp\!\left(\!-\frac{d^2}{2 k \|\vec{x}\|^2}\!\right)
      \min\!\left\{\!1, \frac{k \|\vec{x}\|^2}{d^2}\!\right\} -
  C_4 \sqrt{\frac{\ln(1 / \gamma)}{k}}
  \right)\!
  \right)\;,\]
where $C_1$, $C_2$, $C_3$ and~$C_4$ are positive constants defined as in the proof of Theorem~\ref{th:random}.

On the other hand, for $(\vec{x}, y)$ sampled from the $(\vec{\varphi}, \varrho, \varsigma)$-Gaussian distribution, \citet[Lemma~13]{SchmidtSTTM18} tell us that with probability at least $1 - \gamma^\dag$ one has
\[\|\vec{x}\| <
  \varrho + \varsigma (\sqrt{d} + \sqrt{2 \ln(1 / \gamma^\dag)})\;,\]
and again \citet[Lemma~15]{SchmidtSTTM18} tell us that with probability at least $1 - \gamma^\ddag$ one has
\[y \vec{\varphi}^\top \vec{x} >
  \varrho - \varsigma \sqrt{2 \ln(1 / \gamma^\ddag)}\;.\]

To complete the proof, we take $\gamma^\dag = \gamma^\ddag$ and combine the inequalities above.
\end{proof}

Next we obtain a counterpart of Corollary~\ref{cor:random}.  It shows that, for sufficiently large input dimensions~$d$, and under relatively mild restrictions on the growth rates of the network width~$k$, the mean radius~$\varrho$ and the variance parameter~$\varsigma$ of the Gaussian data model, expected reprogramming accuracy can be arbitrarily close to~$100\%$.  Particularly interesting is the constraint $\eta_{(\varrho)} - \eta_{(\varsigma)} > 1 / 2 - \eta_{(k)} / 2$, which amounts to $\sqrt{d} \varsigma / \varrho = o(\sqrt{k})$, i.e.\ the adversarial task is required to be not too difficult, where greater difficulty may be allowed by greater network width.

\begin{corollary}
\label{cor:G.random}
Suppose that
\[k         = \Theta(d^{\eta_{(k)}})\;, \quad
  \varrho   = \Theta(d^{\eta_{(\varrho)}}) \quad\text{and}\quad
  \varsigma = O(d^{\eta_{(\varsigma)}})\;,\]
where $\eta_{(k)}, \eta_{(\varrho)}, \eta_{(\varsigma)} \in [0, 1]$ are arbitrary constants that satisfy
\[\eta_{(k)}                            > 0\;, \quad
  \eta_{(\varrho)}                      < 1 - \eta_{(k)} / 2\;, \quad
  \eta_{(\varsigma)}                    < 1 / 2 - \eta_{(k)} / 2 \quad\text{and}\quad
  \eta_{(\varrho)} - \eta_{(\varsigma)} > 1 / 2 - \eta_{(k)} / 2\;.\]
Then, for sufficiently large input dimensions~$d$, the expected accuracy of the adversarially reprogrammed network~$\mathcal{N}$ on the $(\vec{\varphi}, \varrho, \varsigma)$-Gaussian data model is arbitrarily close to~$100\%$.
\end{corollary}

\begin{proof}
Fix~$\gamma$ and~$\gamma^\dag$ such that $(1 - C_1 \gamma) (1 - 2 \gamma^\dag)$ is as close to~$100\%$ as required.

By the assumptions in the statement, we have
\[\varrho / \varsigma =
  \Omega(d^{\eta_{(\varrho)}}) / O(d^{\eta_{(\varsigma)}}) =
  \Omega(d^{\eta_{(\varrho)} - \eta_{(\varsigma)}}) =
  \omega_d(1)\;,\]
so $\varrho^2 / \varsigma^2 \geq 2 \ln(1 / \gamma^\dag)$ for large enough~$d$; moreover, we have
\[\beta = \Omega(d^{\eta_{(\varrho)}})\;, \quad
  \beta' = O(d^{\max\{\eta_{(\varrho)},
                      \eta_{(\varsigma)} + 1 / 2\}}) \quad\text{and}\quad
  \exp\!\left(\!\frac{d^2}{k {\beta'}^2}\!\right) = \omega(d^2)\;.\]

The corollary follows by applying Theorem~\ref{th:G.random} and observing that
\begin{multline*}
\frac{C_2}{2} \beta -
\beta' \!
\left(\!
C_3 \exp\!\left(\!-\frac{d^2}{2 k {\beta'}^2}\!\right)
    \min\!\left\{\!1, \frac{k {\beta'}^2}{d^2}\!\right\} +
C_4 \sqrt{\frac{\ln(1 / \gamma)}{k}}
\right) \\
=
\Omega(d^{\eta_{(\varrho)}}) -
O(d^{\max\{\eta_{(\varrho)},
           \eta_{(\varsigma)} + 1 / 2\}}) \!
\left(\!
\frac{1}{\omega(d) \ln(\omega(d^2))} +
\frac{1}{\Omega(d^{\eta_{(k)} / 2})}
\!\right) \\
=
\Omega(d^{\eta_{(\varrho)}}) -
O(d^{\max\{\eta_{(\varrho)},
           \eta_{(\varsigma)} + 1 / 2\} -
     \eta_{(k)} / 2})
=
\Omega(d^{\eta_{(\varrho)}}) -
o(d^{\eta_{(\varrho)}})
=
\Omega(d^{\eta_{(\varrho)}})\;,
\end{multline*}
which is positive for large enough~$d$.
\end{proof}

\subsection{Implicit bias}

We consider, as in Section~\ref{s:bias}, a trajectory of gradient flow for a two-layer ReLU network~$\mathcal{N}_{\vec{\theta}}$ with input dimension~$d$ and width~$k$, trained on an orthogonally separable binary classification dataset, from a balanced and live initialisation, using either the exponential or the logistic loss function.  Also as in Section~\ref{s:bias}, the assumption $k \leq d$ is not needed here.

Hence Corollary~\ref{cor:bias} still applies, and we restate here for convenience.  Recall that $\vec{w}_j$~are first-layer network weights, $a_j$~are second-layer network weights, and $\vec{v}_1$~and $\vec{v}_{-1}$~are the maximum-margin vectors for the positive and negative data classes (respectively).

\firstcorbias*

By adapting the proof of Proposition~\ref{pr:fail} to the Gaussian adversarial task, we obtain the following result about consequences for adversarial reprogramming of the long training of the network.  It tells us that, for any class label mapping from the original to the adversarial task, for any Gaussian data model whose mean direction is in a half-space of the difference of the maximum-margin vectors, and for any adversarial program, the accuracy tends to~$1 / 2$ provided that $\varrho / \varsigma$ tends to infinity, i.e.\ the difficulty $\sqrt{d} \varsigma / \varrho$ of the data model increases slower than the square root of the input dimension~$d$.  The latter constraint mirrors the corresponding conclusion for the Bernoulli data model in Proposition~\ref{pr:fail}, and is again considerably weaker than in the results of \citet{SchmidtSTTM18} on the Gaussian data model, where $\varrho = \sqrt{d}$ and $\varsigma = O(\!\sqrt[4]{d})$ are assumed.

\begin{proposition}
\label{pr:G.fail}
Suppose network weights~$\vec{\theta}$ are is in Corollary~\ref{cor:bias}, class label mapping $m \in \{\pm 1\}$ is arbitrary, data model~$\mathcal{D}$ is any $(\vec{\varphi}, \varrho, \varsigma)$-Gaussian distribution such that $m \cos \angle(\vec{v}_1 - \vec{v}_{-1}, \vec{\varphi}) < 0$, and adversarial program~$\vec{p}$ is arbitrary.  Then we have that
\[\mathbb{P}_{(\vec{x}, y) \sim \mathcal{D}}
  \{m \, y \, \mathcal{N}_{\vec{\theta}}(\vec{p} + \vec{x}) > 0\} \,\leq\,
  \frac{1}{2} +
  \frac{1}{2}
  \exp\!\left(-\frac{\varrho^2 \cos^2 \angle(\vec{v}_1 - \vec{v}_{-1}, \vec{\varphi})}
                    {2 \varsigma^2}\right)\;.\]
\end{proposition}

\begin{proof}
The argument is as in the proof of Proposition~\ref{pr:fail}, except that we invoke \citet[Lemma~17]{SchmidtSTTM18} to show that $\cos \angle(\vec{v}_1 - \vec{v}_{-1}, \vec{\varphi}) < 0$ implies
\[\mathbb{P}_{(\vec{x}, 1) \sim \mathcal{D}}\!
  \left\{\vec{v}_1^\top \vec{x} >
         \vec{v}_{-1}^\top \vec{x}\right\}
  \leq
  \exp\!\left(-\frac{\varrho^2 \cos^2 \angle(\vec{v}_1 - \vec{v}_{-1}, \vec{\varphi})}
                    {2 \varsigma^2}\right)\;.
  \mbox{\qedhere}\]
\end{proof}

\section{Further details on experiments}

Here we supplement Section~\ref{s:exper} by presenting experimental results on adversarial tasks other than MNIST, and listing the data behind the test accuracy plots.

\subsection{Experiments using other datasets}
\label{app:other}

We repeated our experiments with the Fashion-MNIST~\citep{XiaoRV17} dataset as the adversarial task instead of MNIST, while keeping the rest of the experimental setup unchanged.  Fashion-MNIST also consists of $60,\!000$ training images and $10,\!000$ test images, and it is available under the MIT licence.  It has $10$~labels: $0$~for T-shirt/top, $1$~for Trouser, $2$~for Pullover, etc.

The average test accuracies are plotted in Figure~\ref{f:experimentsfashion}.  The compute involved is similar to that for MNIST, which we discussed in Section~\ref{s:exper}.

We then repeated our experiments with the Kuzushiji-MNIST~\citep{ClanuwatBKLYH18} dataset as the adversarial task, again keeping the rest of the experimental setup unchanged.  Kuzushiji-MNIST also consists of $60,\!000$ training images and $10,\!000$ test images, and it is available under the Creative Commons Attribution Share-Alike 4.0 licence.  It consists of $10$~classes of handwritten Japanese characters.

The average test accuracies are plotted in Figure~\ref{f:experimentskmnist}, and the compute involved is again similar to that for MNIST.

\begin{figure}[p]
    \centering
    \begin{tikzpicture}
        \begin{groupplot}[group style={group size= 2 by 1},height=5.9cm,width=7.3cm]
            \nextgroupplot[xlabel=$r$, ylabel=test accuracy, xmode=log,minor tick num=1,ymin=0, ymax=1, xmax=1,log ticks with fixed point, ymajorgrids=true,  yminorgrids=true, minor grid style={line width=.2pt,draw=gray!50,legend pos=outer north east},title=Scheme 1]

            \errorbarplot{resnet50}{fashion_mnist}{r}
            \errorbarplot{ResNet50V2}{fashion_mnist}{r}
            \errorbarplot{ResNet101V2}{fashion_mnist}{r}
            \errorbarplot{inception-v3}{fashion_mnist}{r}
            \errorbarplot{EfficientNetB0}{fashion_mnist}{r}
            \errorbarplot{ResNet152V2}{fashion_mnist}{r}

            \coordinate (left) at (rel axis cs:0,1);

            \nextgroupplot[xlabel=$v$, xmode=log,minor tick num=1,ymin=0, ymax=1, xmax=1,log ticks with fixed point, ymajorgrids=true,  yminorgrids=true, minor grid style={line width=.2pt,draw=gray!50,legend pos=outer north east},title=Scheme 2,legend columns=3,legend to name=networksfashion]

            \errorbarplot{resnet50}{fashion_mnist}{v}\addlegendentry{ResNet-50}
            \errorbarplot{ResNet50V2}{fashion_mnist}{v}\addlegendentry{ResNet-50V2}
            \errorbarplot{ResNet101V2}{fashion_mnist}{v}\addlegendentry{ResNet-101V2}
            \errorbarplot{inception-v3}{fashion_mnist}{v}\addlegendentry{Inception-v3}
\errorbarplot{EfficientNetB0}{fashion_mnist}{v}\addlegendentry{EfficientNet-B0}
\errorbarplot{ResNet152V2}{fashion_mnist}{v}\addlegendentry{ResNet-152V2}

            \coordinate (right) at (rel axis cs:1,1);

        \end{groupplot}
        \coordinate (middle) at ($(left)!.5!(right)$);
        \node[below] at (middle |- current bounding box.south) {\pgfplotslegendfromname{networksfashion}};
    \end{tikzpicture}
    \caption{The accuracy achieved by the adversarial program on the Fashion-MNIST test set for different parameters of the two schemes of combining input images with adversarial programs. The horizontal axes are logarithmic. The values plotted are averages over $5$~trials, which are listed together with the standard deviations in \appendixref{app:data}.}
    \label{f:experimentsfashion}
\end{figure}
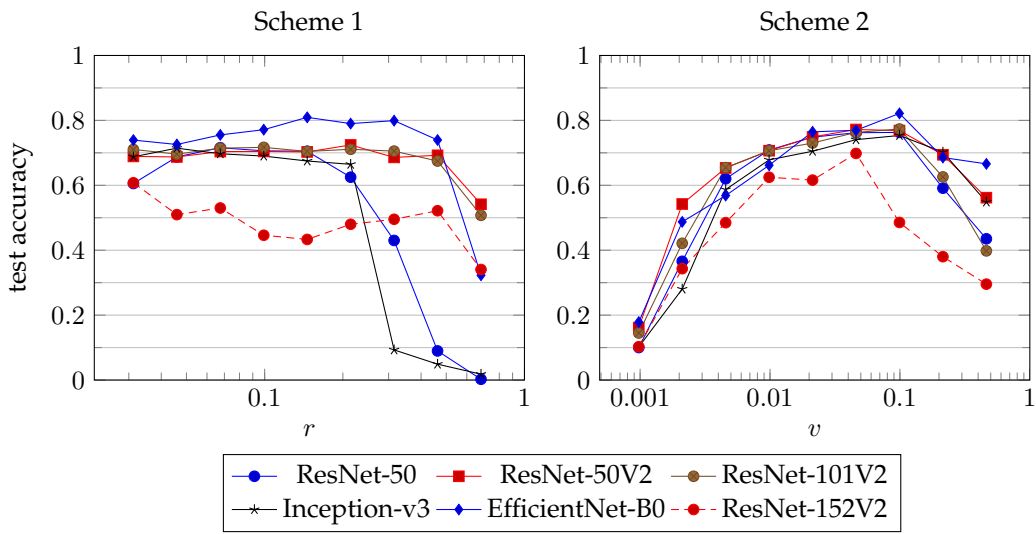

\begin{figure}[p]
    \centering
    \begin{tikzpicture}
        \begin{groupplot}[group style={group size= 2 by 1},height=5.9cm,width=7.3cm]
            \nextgroupplot[xlabel=$r$, ylabel=test accuracy, xmode=log,minor tick num=1,ymin=0, ymax=1, xmax=1,log ticks with fixed point, ymajorgrids=true,  yminorgrids=true, minor grid style={line width=.2pt,draw=gray!50,legend pos=outer north east},title=Scheme 1]

            \errorbarplot{resnet50}{kmnist}{r}
            \errorbarplot{ResNet50V2}{kmnist}{r}
            \errorbarplot{ResNet101V2}{kmnist}{r}
            \errorbarplot{inception-v3}{kmnist}{r}
            \errorbarplot{EfficientNetB0}{kmnist}{r}
            \errorbarplot{ResNet152V2}{kmnist}{r}

            \coordinate (left) at (rel axis cs:0,1);

            \nextgroupplot[xlabel=$v$, xmode=log,minor tick num=1,ymin=0, ymax=1, xmax=1,log ticks with fixed point, ymajorgrids=true,  yminorgrids=true, minor grid style={line width=.2pt,draw=gray!50,legend pos=outer north east},title=Scheme 2,legend columns=3,legend to name=networkskmnist]

            \errorbarplot{resnet50}{kmnist}{v}\addlegendentry{ResNet-50}
            \errorbarplot{ResNet50V2}{kmnist}{v}\addlegendentry{ResNet-50V2}
            \errorbarplot{ResNet101V2}{kmnist}{v}\addlegendentry{ResNet-101V2}
            \errorbarplot{inception-v3}{kmnist}{v}\addlegendentry{Inception-v3}
\errorbarplot{EfficientNetB0}{kmnist}{v}\addlegendentry{EfficientNet-B0}
\errorbarplot{ResNet152V2}{kmnist}{v}\addlegendentry{ResNet-152V2}

            \coordinate (right) at (rel axis cs:1,1);

        \end{groupplot}
        \coordinate (middle) at ($(left)!.5!(right)$);
        \node[below] at (middle |- current bounding box.south) {\pgfplotslegendfromname{networkskmnist}};
    \end{tikzpicture}
    \caption{The accuracy achieved by the adversarial program on the Kuzushiji-MNIST test set for different parameters of the two schemes of combining input images with adversarial programs. The horizontal axes are logarithmic. The values plotted are averages over $5$~trials, which are listed together with the standard deviations in \appendixref{app:data}.}
    \label{f:experimentskmnist}
\end{figure}
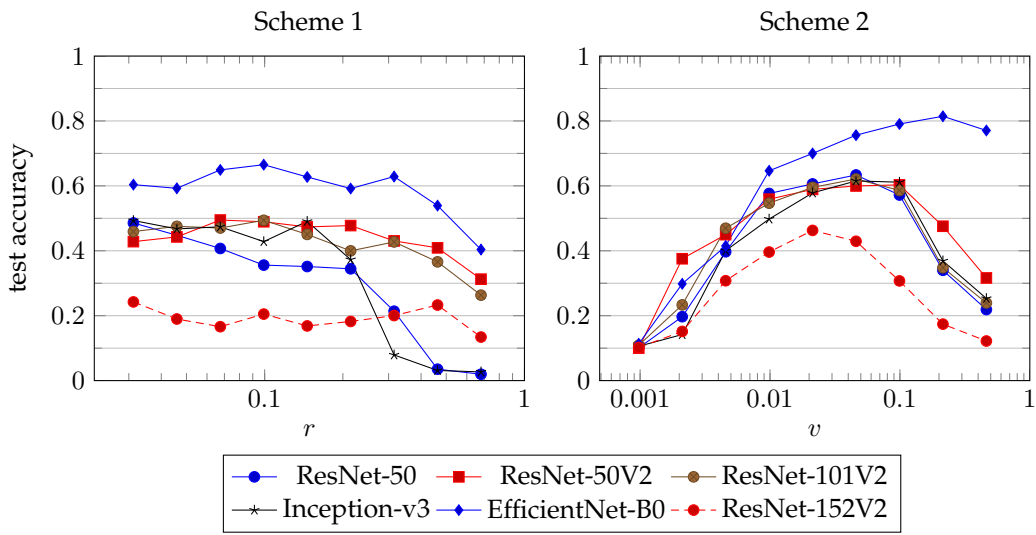

\subsection{Data from experiments}
\label{app:data}

Please see Tables \ref{t:mnistr}, \ref{t:mnistv}, \ref{t:fashionr}, \ref{t:fashionv}, \ref{t:kmnistr}, and~\ref{t:kmnistv}.

We remark that, in the cases where the standard deviation is relatively large, that is due to some of the trials failing to find an adversarial program with non-trivial test accuracy, and the remaining trials succeeding.  It might be interesting in future work to investigate this apparent bimodality.

\setlength{\tabcolsep}{0.2em}
\newlength{\expsep}
\setlength{\expsep}{1.9em}

\newcommand{\expdata}[5]
{\begin{table}[p]
\caption{Accuracy averages~($\%$) on the #5 test set, which are plotted in Figure~\ref{f:experiments#1} on the #3, with standard deviations~($\%$) shown in parentheses, over $5$~trials.}
\label{t:#1#4}
\vspace{1.5ex}
\centering
\small
\begin{tabular}{l@{\hspace{\expsep}}rr@{\hspace{\expsep}}rr@{\hspace{\expsep}}rr@{\hspace{\expsep}}rr@{\hspace{\expsep}}rr@{\hspace{\expsep}}rr} \toprule
&
\multicolumn{2}{r@{\hspace{\expsep}}}{Efficient} &
\multicolumn{2}{r@{\hspace{\expsep}}}{Inception-} &
\multicolumn{2}{r@{\hspace{\expsep}}}{ResNet-} &
\multicolumn{2}{r@{\hspace{\expsep}}}{ResNet-} &
\multicolumn{2}{r@{\hspace{\expsep}}}{ResNet-} &
\multicolumn{2}{r}{ResNet-} \\
$#4$ &
\multicolumn{2}{r@{\hspace{\expsep}}}{Net-B0} &
\multicolumn{2}{r@{\hspace{\expsep}}}{v3} &
\multicolumn{2}{r@{\hspace{\expsep}}}{101V2} &
\multicolumn{2}{r@{\hspace{\expsep}}}{152V2} &
\multicolumn{2}{r@{\hspace{\expsep}}}{50} &
\multicolumn{2}{r}{50V2} \\ \midrule
\csvreader[column names={#4-sym=\sym,
                         val-acc-mean-EfficientNetB0=\efficientmean, val-acc-std-EfficientNetB0=\efficientstd,
                         val-acc-mean-ResNet101V2=\resnetmiddlemean, val-acc-std-ResNet101V2=\resnetmiddlestd,
                         val-acc-mean-ResNet50V2=\resnetsmallmean,   val-acc-std-ResNet50V2=\resnetsmallstd,
                         val-acc-mean-inception-v3=\inceptionmean,   val-acc-std-inception-v3=\inceptionstd,
                         val-acc-mean-resnet50=\resnetoldmean,       val-acc-std-resnet50=\resnetoldstd,
                         val-acc-mean-ResNet152V2=\resnetbigmean,    val-acc-std-ResNet152V2=\resnetbigstd,
                         val-acc-count-EfficientNetB0=\efficientcount,
                         val-acc-count-ResNet101V2=\resnetmiddlecount,
                         val-acc-count-ResNet50V2=\resnetsmallcount,
                         val-acc-count-inception-v3=\inceptioncount,
                         val-acc-count-resnet50=\resnetoldcount,
                         val-acc-count-ResNet152V2=\resnetbigcount},
           late after line=\\,
           late after last line=\\ \bottomrule]
          {aggregate_#2_#4.csv}{}
          {\sym &
           \percentage{\efficientmean}    & (\percentage{\efficientstd}) &
           \percentage{\inceptionmean}    & (\percentage{\inceptionstd}) &
           \percentage{\resnetmiddlemean} & (\percentage{\resnetmiddlestd}) &
           \percentage{\resnetbigmean}    & (\percentage{\resnetbigstd}) &
           \percentage{\resnetoldmean}    & (\percentage{\resnetoldstd}) &
           \percentage{\resnetsmallmean}  & (\percentage{\resnetsmallstd})}
\end{tabular}
\end{table}}

\expdata{mnist}{mnist}{left}{r}{MNIST}
\expdata{mnist}{mnist}{right}{v}{MNIST}
\expdata{fashion}{fashion_mnist}{left}{r}{Fashion-MNIST}
\expdata{fashion}{fashion_mnist}{right}{v}{Fashion-MNIST}
\expdata{kmnist}{kmnist}{left}{r}{Kuzushiji-MNIST}
\expdata{kmnist}{kmnist}{right}{v}{Kuzushiji-MNIST}

\else\fi
\end{document}